\documentclass{article}

\usepackage[margin=1in]{geometry}
\usepackage{amsmath,amsthm,amssymb}
\usepackage{commath,mathtools}
\usepackage{algorithm,algorithmic}
\usepackage{graphicx}
\usepackage{subfigure}
\usepackage{hyperref}
\usepackage{color}
\usepackage{rotating,booktabs}
\usepackage{lscape}
\usepackage[sort,nocompress]{cite}
\usepackage{multirow}
\usepackage{float}
\usepackage[normalem]{ulem}
\newtheorem{theorem}{Theorem}[section]

\newtheorem{lemma}[theorem]{Lemma}
\newtheorem{definition}{Definition}[section]
\newtheorem*{remark}{Remarks}

\usepackage[pagewise,mathlines]{lineno}

\newcommand{\bbf}{\mathbf{b}}
\newcommand{\cbf}{\mathbf{c}}
\newcommand{\xbf}{\mathbf{x}}
\newcommand{\ybf}{\mathbf{y}}
\newcommand{\zbf}{\mathbf{z}}
\newcommand{\gbf}{\mathbf{g}}
\newcommand{\ubf}{\mathbf{u}}
\newcommand{\vbf}{\mathbf{v}}
\newcommand{\wbf}{\mathbf{w}}

\newcommand{\Abf}{\mathbf{A}}
\newcommand{\Wbf}{\mathbf{W}}
\newcommand{\Xbf}{\mathbf{X}}
\newcommand{\Bbf}{\mathbf{B}}
\newcommand{\reps}{r_{\varepsilon}}

\newcommand{\xk}{\mathbf{x}_k}

\newcommand{\xkp}{\mathbf{x}_{k+1}}
\newcommand{\zkp}{\mathbf{z}_{k+1}}
\newcommand{\ukp}{\mathbf{u}_{k+1}}
\newcommand{\vkp}{\mathbf{v}_{k+1}}

\newcommand{\Ccal}{\mathcal{C}}
\newcommand{\Scal}{\mathcal{S}}
\newcommand{\Xcal}{\mathcal{X}}
\newcommand{\Ycal}{\mathcal{Y}}

\newcommand{\hbf}{\mathbf{h}}
\newcommand{\gi}{\gbf_i}
\newcommand{\epsk}{\varepsilon_{k}}
\newcommand{\epskp}{\varepsilon_{k+1}}

\newcommand{\pkp}{\mathbf{p}_{k+1}}
\newcommand{\kl}{k_{l}}
\newcommand{\klp}{k_{l+1}}
\newcommand{\Lepsk}{L_{\varepsilon_{k}}}

\newcommand{\sji}{\mathbf{s}_{j,i}}

\newcommand{\xjp}{\mathbf{x}_{j+1}}
\newcommand{\epsj}{\varepsilon_{j}}
\newcommand{\xhat}{\hat{\mathbf{x}}}

\newcommand{\etol}{\epsilon_{\mathrm{tol}}}

\DeclareMathOperator{\dist}{\mathrm{dist}}

\DeclareMathOperator*{\argmin}{\mathrm{arg\,min}}
\DeclareMathOperator*{\argmax}{\mathrm{arg\,max}}

\newcommand{\blue}[1]{{\color{black}{#1}}}

\newcommand{\newblue}[1]{{\color{black}{#1}}}

\title{Learnable Descent Algorithm for Nonsmooth Nonconvex Image Reconstruction}
\author{
Yunmei Chen\thanks{Department of Mathematics, University of Florida, Gainesville, FL 32611, USA (\texttt{yun@math.ufl.edu}).} \and
Hongcheng Liu\thanks{Industrial and Systems Engineering, University of Florida, Gainesville, FL 32611, USA (\texttt{liu.h@ufl.edu}).} \and
Xiaojing Ye\thanks{Department of Mathematics and Statistics, Georgia State University, Atlanta, GA 30303, USA (\texttt{xye@gsu.edu}).} \and
Qingchao Zhang\thanks{Department of Mathematics, University of Florida, Gainesville, FL 32611, USA (\texttt{qingchaozhang@ufl.edu}).} \and  
}
\date{}

\begin{document}

\maketitle

\begin{abstract}
We propose a general learning based framework for solving nonsmooth and nonconvex image reconstruction problems. We model the regularization function as the composition of the $l_{2,1}$ norm and a smooth but nonconvex feature mapping parametrized as a deep convolutional neural network. We develop a descent-type algorithm to solve the nonsmooth nonconvex minimization problem by leveraging the Nesterov's smoothing technique and the idea of residual learning, and learn the network parameters such that the outputs of the algorithm match the references in training data. Our method is versatile as one can employ various modern network structures into the regularization, and the resulting network inherits the convergence properties of the algorithm. We also show that the proposed network is parameter-efficient and its performance compares favorably to the state-of-the-art methods in a variety of image reconstruction problems in practice.
\end{abstract}

\section{Introduction}
\label{sec:intro}

In the past several decades, variational methods and optimization techniques have been extensively studied for solving image reconstruction problems. For example, a number of regularizers, including total variation (TV) \cite{rudin1992nonlinear}, nonlocal TV \cite{buades2005non,buades2010image},
generalized TV \cite{bredies2010total}, and minmax-concave TV \cite{du2018minmax}, have been proposed to improve the classical Tikhonov-type regularizers in image reconstruction. Advanced optimization techniques were also developed to solve these nonsmooth and/or nonconvex reconstruction models for better computational efficiency, often by leveraging the special structures of the regularizers. However, the image reconstruction quality heavily depends on these hand-crafted regularizers, which are still overly simple and incapable to capture the complex structural features of images. Moreover, the slow convergence and subtle parameter tuning of the optimization algorithms have hindered their applications in real-world image reconstruction problems.

Recent years have witnessed the tremendous success of deep learning in a large variety of real-world application fields \cite{DLH14,HSW89,LPW17,Yarotsky}.
At the heart of deep learning are the deep neural networks (DNNs) which have provable representation power and the substantial amount of data available nowadays for training these DNNs. 
Deep learning was mostly used as a data-driven approach since the DNNs can be trained with little or no knowledge about the underlying functions to be approximated.
However, there are several major issues of such standard deep learning approaches: (i) Generic DNNs may fail to approximate the desired functions if the training data is scarce; (ii) The training of these DNNs are prone to overfitting, noises, and outliers; and (iii) The trained DNNs are mostly ``blackboxes'' without rigorous mathematical justification and can be very difficult to interpret. 

\subsection{Background of learnable optimization algorithms}
To mitigate the aforementioned issues of DNNs, a class of \emph{learnable optimization algorithms} (LOAs) has been proposed recently.
The main idea of LOA is to map a known iterative optimization algorithm to a DNN. The DNN is restricted to have a small number of blocks, where each block (also called a phase) mimics one iteration of the algorithm but with certain components replaced by network layers, and the network parameter $\theta$ of the DNN is learned such that the outputs of the DNN fit the desired solutions given in the training data.

Consider the standard setting of supervised learning with a pair of training data $(\bbf, \hat{\xbf})$, where $\bbf$ is the input data of the DNN, for instance, a noisy and/or corrupted image or compressed or encoded data, and $\hat{\xbf}$ is the corresponding ground truth high quality image that the output of the DNN is expected to match.
\blue{
We study a general framework of LOA which can be described by the following problem:
%
%
%
\begin{equation}\label{eq:loa}
\mathbf{x}_{\theta} = \argmin_{\xbf \in \Xcal}\, \{\phi(\mathbf{x}; \bbf, \theta) := f(\xbf; \bbf) + r(\xbf; \theta) \},
\end{equation}
where $f$ is the data fidelity term to ensure that the reconstructed image $\mathbf{x}$ is faithful to the given data $\mathbf{b}$, $r$ is the regularization that may incorporate proper prior information of $\mathbf{x}$, and $\Xcal \subset \mathbb{R}^n$ is the admissible set of solutions.} 
The regularization $r(\cdot; \theta)$ is realized as a DNN with parameter $\theta$ to be learned. 
\blue{If needed, the data fidelity term $f$ can also incorporate learnable components, in which case the parameters in $r$ and $f$ are collectively denoted by $\theta$. However, in the present work, we consider the case where only $r$ involves $\theta$ for simplicity, as the generalization is straightforward.}

\blue{
The optimal parameter $\theta$ is obtained by minimizing the loss function, $\mathcal{L}(\mathbf{x}_{\theta}, \hat{\xbf})$, which measures the difference between $\mathbf{x}_{\theta}$ and the ground truth reference image $\hat{\xbf}$ corresponding to the data $\bbf$. A typical choice is $\mathcal{L}(\mathbf{x}_{\theta}, \hat{\xbf})= (1/2)\cdot \| \mathbf{x}_{\theta}- \hat{\xbf} \|^2$. }
In practice, we are provided a training set of $N$ pairs $\{(\hat{\xbf}^{(s)},\bbf^{(s)}): s \in [N]\}$, so that \eqref{eq:loa} can be solved for $N$ instances with the shared parameter $\theta$. In this case, each data pair $(\hat{\xbf}^{(s)},\bbf^{(s)})$ yields a solution $\xbf_{\theta}^{(s)}$, and the total loss function can be set to the average $(1/N)\cdot \sum_{s=1}^{N} \mathcal{L}(\xbf_{\theta}^{(s)},\hat{\xbf}^{(s)})$. 
%

%
\blue{The training of $\theta$ can be cast as a bi-level optimization
that minimizes $(1/N)\cdot \sum_{s=1}^{N} \mathcal{L}(\xbf_{\theta}^{(s)},\hat{\xbf}^{(s)})$ with respect to $\theta$, subject to the constraint \eqref{eq:loa} which is called the lower-level problem. However, such bi-level problems are generally very challenging to solve.
A typical workaround is to approximate the actual minimizer $\xbf_{\theta}$ of \eqref{eq:loa} by the output of a LOA-based DNN (different from the DNN $r$) which mimics an iterative optimization scheme for solving the lower-level minimization in the constraint of \eqref{eq:loa} with a small number of iterations (e.g., 15).
In this case, the bi-level optimization is effectively a standard optimization as the approximate solution $\xbf_{\theta}$ is explicitly dependent on the parameter $\theta$; and the optimal $\theta$ can be obtained by applying (stochastic) gradient descent based algorithms, such as ADAM \cite{kingma2014adam}, to minimize the total loss function with respect to $\theta$, as in standard deep learning methods. Thus the trained LOA can output high-quality image with only a small number of iterations in practice. This is the most significant advantage of the deep learning based methods and can greatly reduce computational time in practice.}
LOA is also widely known as the \emph{unrolling method}, as the iteration scheme of the optimization algorithm is ``unrolled'' into multiple blocks of the LOA-based DNN.
%
%
%
However, despite of their promising performance in practice, a large number of existing LOA-based DNNs only superficially resemble the steps of optimization algorithms, and hence they do not really yield a convergent algorithm or correspond to solving any interpretable variational model as the one in \eqref{eq:loa}. As a result, they lack theoretical justifications and convergence guarantees.
\blue{It is worth noting that, although training the parameter $\theta$ in \eqref{eq:loa} is not a bi-level problem with the aforementioned finite-iteration approximation, LOA is not a standalone optimization algorithm neither. This is because that a LOA should not only have theoretical convergence guarantee and high empirical performance in solving (1), but also allow efficient training of the network parameter $\theta$ using training data, as in the upper problem of the bi-level formulation. This work is aimed at designing such a LOA with provable convergence and iteration complexity for solving \eqref{eq:loa} with any fixed $\theta$, whereas the optimal network parameter $\theta$ is still obtained by ADAM as in standard deep network training.}

\subsection{Our goal and approach}

Our goal in this work is to develop a general LOA framework for solving nonsmooth and nonconvex image reconstruction problems. More precisely, we will develop a novel algorithm for solving
\eqref{eq:loa}, where the regularization function is a learnable nonsmooth and nonconvex function. The proposed LOA for  has the following perperties: \emph{Versatility}---our method is flexible and allows users to plug in various kinds of deep neural networks for learning the objective function; \emph{Convergence}---we can ensure convergence of our network with well-trained parameters as its architecture follows exactly the proposed algorithm for solving the nonsmooth and nonconvex optimization in \eqref{eq:loa}; and \emph{Performance}---our method can adaptively learn the regularization function from the training data, such that it is competitive and can even outperform the state-of-the-art methods in terms of both reconstruction accuracy and efficiency in practice.

To this end, we consider to learn the minimization \eqref{eq:loa} for image reconstruction (we assume $\Xcal=\mathbb{R}^n$ for simplicity throughout this work), such that its solution is close to the ground truth high quality image in the training data. Specifically, we use a composited structure of the regularization $r$ as the $l_{2,1}$ norm of a learnable feature mapping $\gbf$ realized by a deep neural network to enhance the sparsity of the feature map of the underlying image. Both $f$ and $\gbf$ are smooth but (possibly) nonconvex, and the overall objective function is nonsmooth and nonconvex. \blue{It is worth pointing out that, the algorithm for solving \eqref{eq:loa}
determines the architecture of the deep neural network, hence the design of an
LOA  should not only consider the convergence and efficiency for solving (1), but also the ability to assist the training of the parameter $\theta$, i.e., reducing the error in minimizing the loss function.} This work is aimed at designing such a LOA with provable convergence and iteration complexity 

We propose a descent-type algorithm to solve this nonsmooth nonconvex problem as follows: (i) We tackle the nonsmoothness by employing Nesterov's smoothing technique \cite{nesterov2005smooth} with automatic diminishing smoothing effect; (ii) We propose two successive residual-type updates, the first one is on $f$ and the second one on $r$, a key idea proven very effective in deep network training \cite{ResNet}, and compute 
the convex combination of the two updates for the next iteration; and (iii) We employ an iterate selection policy based on objective function value to ensure convergence.
Moreover, we prove that a subsequence generated by the proposed LOA has accumulation points and all of them are Clarke stationary points of the nonsmooth nonconvex problem \eqref{eq:loa}. 


\subsection{Notations and organization}
\label{subsec:notation}
We denote $[n]:=\{1,\dots,n\}$ for $n\in \mathbb{N}$. We use regular lower-case letters to denote scalars and scalar-valued functions, and boldfaced lower-case letters for vectors and vector-valued functions. Unless otherwise noted, all vectors are column vectors. The inner product of two vectors $\xbf$ and $\ybf$ is denoted by $\langle \xbf, \ybf \rangle$, and $\|\xbf\| = \|\xbf\|_2$ stands for the $l_2$ norm of $\xbf$ and $\|\Abf\|$ the induced $l_2$ norm of the matrix $\Abf$, and $\Abf^{\top}$ is the transpose of $\Abf$. For any set $\Scal \subset \mathbb{R}^n$, we denote $\dist(\ybf, \Scal) := \inf\{ \|\ybf - \xbf\| \ \vert \ \xbf \in \Scal \}$, and  $\Scal + \ybf := \{\xbf + \ybf\ \vert \ \xbf \in \Scal\}$. \blue{Also note that $\nabla \gbf(\xbf) \in \mathbb{R}^{d \times n}$ for any $\gbf: \mathbb{R}^n \to \mathbb{R}^d$ and $\xbf \in \mathbb{R}^{n}$.} 

The remainder of this paper is organized as follows. In Section \ref{sec:related}, we review the recent literature on LOA and general nonsmooth nonconvex optimization methods. In Section \ref{sec:proposed}, we present our LOA based on a descent type algorithm to solve the nonsmooth nonconvex image reconstruction problem with comprehensive convergence analysis. In Section \ref{sec:experiment}, we conduct a number of numerical experiments on natural and medical image dataset to show the promising performance of our proposed method. We provide several concluding remarks in Section \ref{sec:conclusion}.

\section{Related Work}
\label{sec:related}

\subsection{Learnable optimization algorithms}

Learnable optimization algorithm (LOA) is a class of methods developed in recent years to imitate the iterations in optimization algorithms as blocks in a deep neural network with certain components replaced by learnable layers.
Existing LOAs can be approximately categorized into two groups.

The first group of LOAs appeared in the literature are motivated by the similarity between the iterative scheme of a traditional optimization algorithm (e.g., proximal gradient algorithm) and a feed forward neural network. 
Provided instances of training data, such as ground truth solutions, a LOA replaces certain components of the optimization algorithm with parameters to be learned from the data.
The pioneer work \cite{KY2010} in this group of LOAs is based on the well-known iterative shrinkage thresholding algorithm (ISTA) for solving the LASSO problem.
In \cite{KY2010}, a learned ISTA network, called LISTA, is proposed to replace $\Phi^{\top
}$ by a weight matrix to be learned from instance data to reduce iteration complexity of the original ISTA. 
The asymptotic linear convergence rate for LISTA is established in \cite{CLWY2018} and \cite{LCW2019}.
Several variants of LISTA were also developed using low rank or group sparsity \cite{SBS2015},  $\ell_0$ minimization \cite{XWG16} and learned approximate message passing \cite{BSR2017}.
The idea of LISTA has been extended to solve composite problems with linear constraints, known as the differentiable linearized alternating direction method of multipliers (D-LADMM) \cite{XWZ2019}.
These LOA methods, however, still employ handcrafted regularization and require closed form solution of the proximal operator of the regularization term.

To improve reconstruction quality, the other group of works follow a different approach by replacing the lower-level minimization in \eqref{eq:loa} with a DNN whose structure is inspired by a numerical optimization algorithm for solving the minimization problem.
For example, recall that the standard proximal gradient method applies a gradient descent step on the smooth function $\nabla f$ at the current iterate, and then the proximal mapping of $r$ to obtain the next iterate. In this case, a LOA can be obtained by replacing the proximal mapping of $r$ with a multilayer perceptrons (MLP), which can be learned using training data.
As such, one avoids explicit formation of the regularization $r$ for \eqref{eq:loa}.
This paradigm has been embedded into half quadratic splitting in DnCNN \cite{zhang2017beyond}, ADMM in \cite{chang2017one,meinhardt2017learning} and primal dual methods in \cite{AO18,LCW2019,meinhardt2017learning,wang2016proximal} to solve the subproblems. 
To improve the generic black-box CNNs above, several LOA methods  are proposed to incorporate certain prior knowledge about the solution in the design of $r$, then unroll numerical optimization algorithms as deep neural networks so as to preserve their efficient structures with proven efficiency, such as the ADMM-Net \cite{NIPS2016_6406}, Variational Network \cite{HKK18} and ISTA-Net \cite{Zhang2018ISTANetIO}.
These methods also prescribe the phase number $K$, and map each iteration of the corresponding numerical algorithm to one phase of the network, and learn specific components of the phases in the network using training data.

Despite of the promising performance in a variety of applications, the LOAs are only related to the original optimization algorithms superficially.
These LOAs themselves do not follow any provably convergent algorithm or correspond the solution of any properly defined variational problem.
Moreover, certain acceleration techniques proven to be useful for numerical optimization algorithms are not effective in their LOA counterparts.
For example, the acceleration approach based on momentum \cite{Nesterov} can significantly improve iteration complexity of traditional (proximal) gradient descent methods, but does not have noticeable improvement when deployed in LOAs.
This can be observed by the similar performance of ISTA-Net \cite{Zhang2018ISTANetIO} and FISTA-Net \cite{zhang2017ista}. 
One possible reason is that the LOA has nonconvex components, for which a linear combination of past iterates is potentially a worse extrapolation point in optimization \cite{LL15}.

In parallel to the development of LOAs, performance of deep networks for image reconstruction is continuously being improved due to various network engineering and training techniques these years.
For example, ISTA-Net$^+$ \cite{Zhang2018ISTANetIO} employs the residual network structure \cite{ResNet} and results in substantially increased reconstruction accuracy over ISTA-Net. 
The residual structure is also shown to improve network performance in a number of recent work, such as ResNet-v2 \cite{he2016identity}, WRN \cite{zagoruyko2016wide}, and ResNeXt \cite{xie2017aggregated}.
In image compression and reconstruction, the learnable sampling module is always implemented as a single convolutional layer without activation \cite{SJZ17, Shi_2019_CVPR, zhou2019multi,ZZG20,zhang2020amp}. Efficient block compressed sensing for high-dimensional data \cite{gan2007block} can be achieved by controlling the convolution kernel size and stride \cite{ SJZ17}.
Joint reconstruction to reduce blocky effects in image compressive sensing is proposed in \cite{SJZ17}, and the learned sampling operator is shown to automatically satisfy the orthogonal property in \cite{ZZG20}.
A multi-channel method is proposed in \cite{zhou2019multi} to elaborately manage the sensing resources by assigning different sampling ratio to image blocks. To obtain any desired sampling ratio, a scalable sampling and reconstruction is achieved through a greedy measurement based selection algorithm in \cite{Shi_2019_CVPR}.

\blue{
\subsection{Learning regularizer for inverse problems in imaging}
Learning regularizer from data in inverse problems, particularly with applications in image reconstruction, has emerged in recent years.
In \cite{li2020nett}, a general framework with data-driven regularizer, called network Tikhonov (NETT), is proposed to obtain nearly data-consistent solution with small regularization value. In this work, two types of learnable regularizers are considered: one is a weighted sum of the compositions of $q$th powered norm ($q>1$) and nonlinear deep neural networks, and the other one is a CNN regularizer. The convergence of the regularized solution to the regularization-minimizing solution is discussed; in particular, convergence rate in terms of the noise level and regularization weight is derived.
Learnable regularizer has also been exploited in the optimal control framework \cite{effland2020variational,kobler2020total}.
In \cite{effland2020variational}, a class of learnable Field-of-Experts regularizer is considered. To train the regularizer, an optimal control framework is adopted, where the gradient flow of the regularized inverse problem is used as the dynamical process. The optimal value of the regularization parameter and the stopping time, which play the role of control, is obtained such that the squared distance between the output and the given ground truth image is minimized. Numerically, a forward Euler discretization of the gradient flow is employed, which yields a static variational network. 
In \cite{kobler2020total}, a total deep variation regularizer is proposed for general linear inverse problems, and a similar optimal control framework is adopted to train the optimal regularization parameter and stopping time. In this work, a sensitivity analysis is also conducted to derive a computable upper error bounds of the proposed method.
In \cite{gilton2019learned}, a learnable patch-based regularizer is proposed for image reconstruction problems to alleviate the overfitting issue and reduce sample complexity of deep learning based approaches in practice.
}


\subsection{Nonsmooth and nonconvex optimization}

Nonsmooth nonconvex optimization has been extensively studied in recent years. One of the most common methods is the proximal gradient method (also known as the forward-backward splitting or FBS) \cite{FM81, ABS13, N13}. Several variants, including the accelerated proximal gradient method \cite{LL15} and the FBS with an inertial force \cite{BCL16, OCB14}, are proposed to improve the convergence rate.
Iteratively reweighted algorithms are developed to iteratively solve a proximal operator problem \cite{GZL13, ODB14, ZK14}. These algorithms are effective when the nonsmooth components involved in the subproblems are simple, i.e., the associated proximal operators have closed-form or are easy to solve. 

There are also a number of optimization algorithms developed for certain structured nonsmooth nonconvex problems. For instance, for a class of composite optimization problems involving $h(\cbf(\xbf))$, where $h$ is convex but nonsmooth, and $\cbf$ is smooth but (possibly) nonconvex, several linearized proximal type algorithms are proposed such that $\cbf$ is approximated by its linearization. This renders a convex subproblem in each iteration, which can be solved with exact \cite{LW16, DL16, OFB19} or inexact \cite{GST19} function and gradient evaluations.

If the problem of nonconvexity is due to the difference of convex (DC) functions, a number of optimization methods known as DC algorithms (DCA) are developed \cite{SOS16, WCP18, BBC19, NOS18}. DCA approximates a nonconvex DC problem by a sequence of convex ones, such that each iteration only involves a convex optimization problem. Several DC decomposition and approximation methods have been introduced for different applications \cite{LP05, PL14}. Recently, the proximal linearized algorithms for DC programming are proposed to iteratively solve the subproblem, where one of the convex components is replaced by its linear approximation together with a proximal term \cite{SOS16, BBC19}. In \cite{WCP18}, extrapolation is integrated into proximal linearized algorithm for possible acceleration of the proximal DCA. In \cite{NOS18}, an inexact  generalized proximal linearized algorithms is developed, where the Euclidean distance is replaced with a quasi distance in the proximal operator, and the proximal point is replaced with an approximate proximal point. However, the subproblem of DCA may not have closed-form solution and thus can still require inner iterations to solve.

The Kurdyka-{\L}ojasiewicz (KL) property has been leveraged to study the convergence in nonconvex nonsmooth optimization \cite{attouch2008alternating, ABR10,ABS13,frankel2015splitting}. The KL property is shown to hold for a large class of nonconvex functions used in practice and has been extensively exploited to analyze the convergence rate of first-order algorithms for nonconvex optimization. \blue{In the present work, we do not require the learnable regularizer to satisfy the KL property, and hence the proposed method and its iteration complexity analysis can be potentially applicable to a larger class of nonsmooth and nonconvex problems.}

To solve general nonconvex and nonsmooth problems, a common approach is using the smoothing technique, possibly in combination with gradient descent and line search strategy; see \cite{R98,ZC09,C12,BC17} and the references therein. The main idea of the smoothing technique is to construct a class of smooth nonconvex problems (e.g., using convolution) to approximate the original problem, where the approximation accuracy (smoothing level) is controlled by a smoothing parameter. Then one can apply gradient descent or projected gradient descent with line search to solve the approximate problem with a fixed smoothing level; then reduce the smoothing parameter and solve the problem again, and so on. 

The descent algorithm developed in this work for nonsmooth and nonconvex optimization is largely inspired by \cite{C12}. However, unlike \cite{C12}, our goal is to construct a deep image reconstruction network in the framework of \eqref{eq:loa}, where (part of) the objective function is unknown and needs to be learned from the training data, such that the trained network has convergence guarantee in theory and compares to the state-of-the-art favorably in reconstruction quality in practice.

\section{LDA and Convergence Analysis} 
\label{sec:proposed}
In this section, we propose a novel learnable descent algorithm (LDA) to solve the nonsmooth and nonconvex optimization problem for image reconstruction:
\begin{equation}\label{eq:phi}
  \min_{\xbf \in \Xcal} \ \phi(\xbf) := f(\xbf) + r(\xbf),
\end{equation}
where $\xbf$ is the image to be reconstructed, $\Xcal$ is the admissible set of $\xbf$, e.g., $\Xcal = \mathbb{R}^n$, $n$ is the number of pixels in $\xbf$, $f$ stands for the data fidelity term of $\xbf$, and $r$ represents the regularization term to be learned.

We leverage the sparse selection property of $l_1$ norm and parametrize the regularization term $r$ as the composition of the $l_{2,1}$ norm and a feature extraction operator $\gbf(\xbf)$ to be learned.
Specifically, $\gbf: \mathbb{R}^n \to \mathbb{R}^{md}$ such that $\gbf(\xbf) = (\gbf_1(\xbf),\dots,\gbf_m(\xbf))$, where $\gbf_i(\xbf) \in \mathbb{R}^{d}$ is the $i$-th feature vector of $\xbf$ for $i=1,\dots,m$. 
That is, we set $r$ in \eqref{eq:phi} to
\begin{equation}\label{eq:r}
r(\xbf) := \|\gbf(\xbf)\|_{2,1} = \sum_{i = 1}^{m} \|\gbf_i(\xbf)\|.
\end{equation}
In this paper, $\gbf$ is realized by a deep neural network whose parameters are learned from training data. 
The regularization $r$ in \eqref{eq:r} can be interpreted as follows:
we learn a smooth nonlinear mapping $\gbf$ to extract spares features of $\xbf$, and apply the $l_{2,1}$-norm which has proven to be a robust and effective sparse feature regularization. 
\blue{In our experiments, we employ the convolutional neural network (CNN) architecture for $\gbf$, which yields a nonlinear mapping that produces a feature vector $\gbf_i(\xbf)$ at each pixel $i$, and thus the $l_{2,1}$ norm promotes group sparsity where each group is represented by $\gbf_i(\xbf)$.}
In addition, we make several assumptions on $f$ and $\gbf$ throughout this work.
\begin{itemize}
    \item \textbf{Assumption 1 (A1)} $f$ is differentiable and (possibly) nonconvex, and $\nabla f$ is $L_f$-Lipschitz continuous.
    \item \textbf{Assumption 2 (A2)} Every component of $\gbf$ is differentiable and (possibly) nonconvex, $\nabla \gbf$ is $L_g$-Lipschitz continuous, and $\sup_{\xbf \in \Xcal} \| \nabla \gbf(\xbf)\| \leq M $ for some constant $M>0$.
    \item \textbf{Assumption 3 (A3)} $\phi$ is coercive, and $\phi^* = \min_{\xbf \in \Xcal} \phi(\xbf) > -\infty$.
\end{itemize}
\begin{remark}
The assumptions (A1)--(A3) are mild for imaging applications. The Lipschitz continuity of $\nabla f$ and $\nabla \gbf$ is standard in optimization and most imaging applications; the smoothness of $\gbf$ and boundedness of $\nabla \gbf$ are satisfied for all standard deep neural networks with smoothly differentiable activation functions such as sigmoid, tanh, and elu;
and the coercivity of $\phi$ generally holds in image reconstruction, e.g., the DC component providing overall image intensity information (e.g., $\|\xbf\|_1$) is contained in the data, and deviation from this value makes $f$ value tend to infinity.
\blue{We also note that the Lipschitz constant of $\nabla \gbf$ depends on the weight parameters and may be large in practice, nevertheless our convergence analysis and iteration complexity have taken this into consideration as shown below.}
\end{remark}

Other than the requirement in (A2), the design of network architecture and the choice of activation functions in $\gbf$ are rather flexible.
A typical choice of $\gbf$ is a convolutional neural network (CNN), which maps an input image $\xbf \in \mathbb{R}^n$ (gray-scale image with $n$ pixels) to a collection of $m$ feature vectors $\{\gbf_i(\xbf): 1 \le i \le m\} \subset \mathbb{R}^d$.
%


\subsection{Smooth Approximation of Nonsmooth Regularization} 
\label{subsec:smooth}
To tackle the nonsmooth and nonconvex regularization term $r(\xbf)$ in \eqref{eq:r}, we first employ Nesterov's smoothing technique \cite{nesterov2005smooth} to smooth the $l_{2,1}$ norm in \eqref{eq:r} for any fixed $\gbf(\xbf)$:
\begin{equation}\label{eq:Rx}
r(\xbf) = \max_{\ybf \in \Ycal}\ \langle \gbf(\xbf), \ybf \rangle,
\end{equation}
where $\ybf \in \Ycal$ is the dual variable, $\Ycal$ is the space defined by
\[
\Ycal := \cbr[1]{ \ybf = (\ybf_1,\dots,\ybf_m) \in \mathbb{R}^{md}\ \vert \ \ybf_i=(y_{i1},\dots,y_{id}) \in \mathbb{R}^{d},\ \| \ybf_i \| \leq 1, \ i \in [m] }.
\]
For any $\varepsilon>0$, we consider the smooth version $\reps$ of $r$ by perturbing the dual form \eqref{eq:Rx} as follows:
\begin{equation}\label{eq:r_eps_def}
\reps(\xbf) = \max_{\ybf \in \Ycal} \ \langle \gbf(\xbf), \ybf \rangle - \frac{\varepsilon}{2} \|\ybf\|^2.
\end{equation}
\blue{Note that \eqref{eq:r_eps_def} is one special form of the Nesterov's smoothing technique. This form is also called the Moreau-Yosida regularization scheme, as discussed by, e.g., \cite{moreau1965proximite,Yosidabook,lemarechal1997practical,mifflin1999properties}.
The convergence analysis below can be easily modified and applied to other forms of the Nesterov-type smoothing schemes.}

Then one can readily show that 
\begin{equation}\label{eq:r_eps_bound}
\reps(\xbf)\leq r(\xbf) \leq \reps(\xbf) +\frac{m\varepsilon}{2},\quad \forall\, \xbf \in \mathbb{R}^n.
\end{equation}
Note that the perturbed dual form in \eqref{eq:r_eps_def} has a closed form solution: denoting 
\begin{equation}\label{eq:y_eps}
\ybf_\varepsilon^* = \argmax_{\ybf \in \Ycal}\ \langle \gbf(\xbf) , \ybf \rangle - \frac{\varepsilon}{2} \| \ybf \|^2,
\end{equation}
then solving \eqref{eq:y_eps}, we obtain the closed form of $\ybf_\varepsilon^*=((\ybf_\varepsilon^*)_1,\dots,(\ybf_\varepsilon^*)_m)$ where
\begin{equation}\label{eq:y_eps_sol}
(\ybf_\varepsilon^*)_i = 
\begin{cases}
\frac{1}{\varepsilon} \gbf_i(\xbf), & \mbox{if} \ \|\gbf_i(\xbf) \| \leq \varepsilon, \\
\frac{\gbf_i(\xbf)}{\|\gbf_i(\xbf)\|}, & \mbox{otherwise},
\end{cases}\qquad \mbox{for}\ i\in [m].
\end{equation}
Plugging \eqref{eq:y_eps_sol} back into \eqref{eq:r_eps_def}, we have

\begin{equation}\label{eq:r_eps}
 \reps(\xbf)  = \sum_{i \in I_0} \frac{1}{2\varepsilon}  \|\gbf_i(\xbf)\|^2 + \sum_{i \in I_1} \del[2]{\|\gbf_i(\xbf)\| - \frac{\varepsilon}{2} },
\end{equation}
where the index set $I_0$ and its complement $I_1$ at $\xbf$ for the given $\gbf$ and $\varepsilon$ are defined by
\begin{equation*}
I_0 = \{ i \in [m] \ \vert \ \|\gbf_i(\xbf)\| \leq \varepsilon \}, \ \ \ I_1 = [m] \setminus I_0.
\end{equation*}
Moreover, it is easy to show from \eqref{eq:r_eps} that
\begin{equation}\label{eq:d_r_eps}
\nabla \reps(\xbf) = \nabla \gbf(\xbf)^{\top} \ybf_{\varepsilon}^* = \sum_{i \in I_0} \nabla \gbf_i(\xbf)^{\top} \frac{\gbf_i(\xbf)}{\varepsilon}   + \sum_{i \in I_1} \nabla \gbf_i(\xbf)^{\top}  \frac{\gbf_i(\xbf)}{\|\gbf_i(\xbf)\|} ,
\end{equation}
where $\nabla \gbf_i(\xbf)\in \mathbb{R}^{d \times n}$ is the Jacobian of $\gbf_i$ at $\xbf$.

The smoothing technique above yields a smooth approximation of the nonsmooth function $r(\xbf)$, which allows for rigorous analysis of iteration complexity and provable asymptotic convergence to the original nonsmooth problem \eqref{eq:phi}, as we will show in Section \ref{subsec:convergence}.

\subsection{Proposed Descent Algorithm} 
\label{subsec:alg}
In this subsection, we propose a novel descent type algorithm for solving the minimization problem \eqref{eq:phi} with the regularization $r$ defined in \eqref{eq:r}.
\blue{We choose this scheme over Lagrangian-type method since the nonlinear equality constraint (e.g., $\ybf_i = \gbf_i(\xbf)$) can be difficult to handle in convergence analysis.}
The main idea of our method is to apply a modified gradient descent algorithm to minimize the objective function $\phi$ with the nonsmooth $r$ replaced by the smooth $\reps$ as follows:
\begin{equation}\label{eq:phi_eps}
\phi_\varepsilon(\xbf) := f(\xbf) + \reps (\xbf),
\end{equation}
with $\varepsilon$ automatically decreasing to $0$ as the iteration progresses.
Note that $\phi_{\varepsilon}$ in \eqref{eq:phi_eps} is differentiable since both $\nabla f$ and $\nabla \reps$ (defined in \eqref{eq:d_r_eps}) exist.
Moreover, $\phi_{\varepsilon}(\xbf) \le \phi(\xbf) \le \phi_{\varepsilon}(\xbf) + \frac{m \varepsilon}{2}$ for any $\xbf \in \Xcal$ due to \eqref{eq:r_eps_bound}.

%
%

%
In light of the substantial improvement in practical performance by ResNet \cite{ResNet}, we choose to split $f$ and $\reps$ and perform two residual type updates as follows: In the $k$-th iteration with $\varepsilon = \varepsilon_k >0$, we first compute
\begin{equation}
    \label{eq:z}
    \zkp = \xk - \alpha_k \nabla f(\xk),
\end{equation}
where $\alpha_k$ is the step size to be specified later.
Then we compute two candidates for $\xkp$, denoted by $\ubf_{k+1}$ and $\vbf_{k+1}$, as follows:
\begin{subequations}
\begin{align}
\ubf_{k+1} &= \argmin_{\xbf}\ \langle \nabla f(\xbf_{k}) , \xbf - \xbf_{k}\rangle + \frac{1}{2\alpha_k} \| \xbf - \xbf _{k} \| ^2 + \langle \nabla r_{\varepsilon_{k}} (\zbf_{k+1}) , \xbf - \zbf_{k+1}\rangle +  \frac{1}{2\beta_k} \| \xbf - \zbf _{k+1} \| ^2, \label{eq:u}  \\
\vbf_{k+1} &= \argmin_{\xbf}\ \langle \nabla f(\xbf_{k}) , \xbf - \xbf_{k}\rangle + \langle \nabla r_{\varepsilon_{k}} (\xbf_{k}) , \xbf - \xbf_{k}\rangle  + \frac{1}{2\alpha_k} \| \xbf - \xbf _{k} \| ^2, \label{eq:v} 
\end{align}
\end{subequations}
%
%
where $\beta_k$ is another step size along with $\alpha_k$.
Note that both minimization problems in \eqref{eq:u} and \eqref{eq:v} have closed form solutions:
\begin{subequations}
\begin{align}
\ubf_{k+1} &= \zbf_{k+1} - \tau_k \nabla r_{\varepsilon_{k}} (\zbf_{k+1}) \label{eq:u_closed} \\
\vbf_{k+1} &= \zbf_{k+1} - \alpha_k \nabla r_{\varepsilon_{k}} (\xbf_{k}) \label{eq:v_closed}    
\end{align}
\end{subequations}
where $\nabla r_{\varepsilon_{k}} $ is defined in \eqref{eq:d_r_eps} and $\tau_k = \frac{\alpha_k \beta_k}{\alpha_k + \beta_k}$.
Then we choose between $\ubf_{k+1}$ and $\vbf_{k+1}$ that has the smaller function value $\phi_{\varepsilon_k}$ to be the next iterate $\xbf_{k+1}$:
\begin{equation}\label{eq:x_select}
\xbf _{k+1} = 
\begin{cases}
\ubf_{k+1} &\text{if $\phi_{\varepsilon_k}(\ubf_{k+1}) \leq \phi_{\varepsilon_k}(\vbf_{k+1})$}, \\
\vbf_{k+1} & \text{otherwise}.
\end{cases}
\end{equation}
This algorithm is summarized in Algorithm \ref{alg:lda}.
Line 7 of Algorithm \ref{alg:lda} presents a \emph{reduction criterion}. That is, if the reduction criterion $\| \nabla \phi_{\epsk}(\xkp) \| < \sigma \gamma \epsk$ is satisfied, then the smoothing parameter $\epsk$ is shrunk by $\gamma \in (0,1)$.
%
%

In Algorithm \ref{alg:lda}, $\ubf_{k+1}$ in \eqref{eq:u} can be considered as the convex combination of two successive residue-type updates: the first update is $\zbf_{k+1} = \xk - \alpha_k \nabla f(\xk)$ as defined by \eqref{eq:z}---a gradient descent of $f$ at $\xk$; the second is $\pkp = \zbf_{k+1} - \beta_k \nabla r_{\epsk}(\zbf_{k+1})$---another gradient descent of $r_{\epsk}$ at $\zbf_{k+1}$; and finally $\ubf_{k+1} = \frac{\beta_k}{\alpha_k+\beta_k} \zbf_{k+1} + \frac{\alpha_k}{\alpha_k + \beta_k} \pkp$---the convex combination of $\zbf_{k+1}$ and $\pkp$. In this case, $f$ and $r_{\epsk}$ are separated so they each can participate in a residual-type update, which is proven very effective for imaging applications \cite{ResNet}.
\blue{The $\ukp$ step can also be viewed as an inexact computation for finding the proximal point  $\ukp= \argmin_{\ubf} \frac{1}{2}\|\ubf-\zkp\|^2+\alpha_k r_{\epsk}(\ubf)$. That is to replace $r_{\epsk}$ with its linear approximation at $\zkp$, so that the $\ukp$ step mimics the residual learning architecture for learning unknown regularizer.} 
Note that the $\vbf_{k+1}$ in \eqref{eq:v} is the standard gradient descent of $\phi_{\epsk}$ at $\xbf$ to safeguard the convergence of the algorithm. 
\blue{We set $\xbf_{k+1}$ to $\ubf_{k+1}$ or $\vbf_{k+1}$ whichever has lower value of $\phi_{\varepsilon_k}$ to encourage reduction of the objective function.}

\blue{We remark that $\ubf_{k+1}$ plays the key role in Algorithm \ref{alg:lda} LDA to attain high efficiency by using the residual type updates on $f$ and $r_{\varepsilon}$ progressively, which avoids vanishing gradient in minimizing the loss function \cite{he2016identity} and improves the efficiency in training the network component of $r_{\varepsilon}$.
Our experimental results showed the proposed method outperforms standard gradient decent method, i.e., $\vbf$-subproblem only (see Section 4.3.1 and the right panel of Figure \ref{fig:6}).}

\blue{The step sizes $\alpha_k$ and $\tau_k$ in Algorithm \ref{alg:lda} (LDA) can be learned together with the network parameter $\theta$ during training for $k \le K$, where $K$ is the prescribed iteration number of LDA in training. For $k > K$ in testing, $\alpha_k$ can be either chosen according to the range in Theorem \ref{thm:eta_shrinking}, or by a standard backtracking such that $\alpha_k^{-1} - L < -e$ for an arbitrary user-chosen number $e>0$ (this backtracking procedure is guaranteed to terminate within finitely many trials). The choice of $\tau_k$ is arbitrary without affecting the convergence analysis, but it may affect the probability that we choose $\ubf_{k+1}$ over $\vbf_{k+1}$ in Step 6 of Algorithm \ref{alg:lda} LDA in practice. We will provide details of algorithmic parameter in Section \ref{sec:experiment}.}

\begin{algorithm}[t]
\caption{Learnable Descent Algorithm (LDA) for the Nonsmooth Nonconvex Problem \eqref{eq:phi}}
\label{alg:lda}
\begin{algorithmic}[1]
\STATE \textbf{Input:} Initial $\xbf_0$, $0<\gamma<1$, and $\varepsilon_0,\sigma>0$. Maximum iteration $K$ or tolerance $\etol>0$.
\FOR{$k=0,1,2,\dots,K$}
\STATE $\zkp = \xk - \alpha_k \nabla f(\xk)$
\STATE $\ukp = \zkp - \tau_k \nabla r_{\varepsilon_{k}} (\zkp)$
\STATE $\vkp = \zkp - \alpha_k \nabla r_{\varepsilon_{k}} (\xk)$
\STATE $\xkp = \begin{cases}
\ukp & \mbox{if}\ \phi_{\varepsilon_k}(\ubf_{k+1}) \leq \phi_{\varepsilon_k}(\vbf_{k+1}) \\
\vkp & \mbox{otherwise}
\end{cases}$
\STATE If $\|\nabla \phi_{\varepsilon_k}(\xkp)\| < \sigma \gamma {\varepsilon_k}$,  set $\varepsilon_{k+1}= \gamma {\varepsilon_k}$;  otherwise, set $\varepsilon_{k+1}={\varepsilon_k}$.
\STATE If $\sigma {\varepsilon_k} < \etol$, terminate.
\ENDFOR
\STATE \textbf{Output:} $\xbf_{k+1}$.
\end{algorithmic}
\end{algorithm}

\subsection{Convergence and Complexity Analysis} 
\label{subsec:convergence}
In this subsection, we provide a comprehensive convergence analysis with iteration complexity of the proposed Algorithm \ref{alg:lda} LDA.
Since the objective function in \eqref{eq:phi} is nonsmooth and nonconvex, we adopt the notion of \emph{Clarke subdifferential} \cite{Clarke83} (also called the \emph{limiting subdifferential} or simply \emph{subdifferential}) to characterize the optimality of solutions.
\begin{definition}[Clarke subdifferential]
\label{def:clarke_subdiff}
Suppose that $f: \mathbb{R}^{n} \rightarrow(-\infty,+\infty]$  is locally Lipschitz. The Clarke subdifferential of $f$  at $\xbf$  is defined as 
\[
\partial f(\xbf) := \cbr[2]{\wbf \in \mathbb{R}^{n}\ \bigg\vert \ \langle \wbf, \vbf \rangle \leq \limsup_{\zbf \rightarrow \xbf,\, t \downarrow 0} \frac{f(\zbf+t \vbf)-f(\zbf)}{t}, \ \ \forall\, \vbf \in \mathbb{R}^{n} }.
\]
\end{definition}
\begin{definition}[Clarke stationary point]
\label{def:clarke_cp}
For a locally Lipschitz function $f$, a point $\xbf \in R^n$ is called a Clarke stationary point of $f$ if $0 \in \partial f(\xbf)$.
\end{definition}
Note that for a differentiable function $f$, there is $\partial f (\xbf)= \{\nabla f(\xbf)\}$. For the nondifferentiable (nonsmooth) function $r$ defined in \eqref{eq:r}, we can also compute its Clarke subdifferential as in the following lemma.
\begin{lemma}\label{lem:r_subdiff}
Let $r(\xbf)$ be defined in \eqref{eq:r}, then the Clarke subdifferential of $r$ at $\xbf$ is
\begin{equation}\label{eq:r_subdiff}
\partial r(\xbf) = \cbr[2]{\sum_{i\in I_0}\nabla \gbf_i(\xbf)^{\top}  \wbf_i + \sum_{i \in I_1}\nabla \gbf_i(\xbf)^{\top}\frac{\gbf_i(\xbf)}{\|\gbf_i(\xbf)\|} \ \bigg\vert \ \wbf_i \in \mathbb{R}^d, \ \|\Pi(\wbf_i; \Ccal(\nabla \gbf_i(\xbf)))\|\leq 1,\ \forall\, i \in I_0 } ,  
\end{equation}
where $I_0=\{i \in [m] \ | \ \|\gbf_i(\xbf) \|= 0 \}$, $I_1=[m] \setminus I_0$, and $\Pi(\wbf;\Ccal(\Abf))$ is the projection of $\wbf$ onto $\Ccal(\Abf)$ which stands for the column space of $\Abf$.
\end{lemma}

\begin{proof} 
We observe that $r(\xbf) = \sum_{i=1}^m r_i(\xbf)$ where $r_i(\xbf):= \| \gbf_i (\xbf)\|$.
Hence we can consider the Clarke subdifferential of each $r_i(\xbf)$.

If $i\in I_0$, then for any $\vbf$, there is
\begin{align*}
    &{\frac{r_i(\zbf + t \vbf) - r_i(\zbf)}{t} - \|\nabla \gbf_i(\xbf)\vbf\| } = {\frac{\|\gbf_i(\zbf + t \vbf)\| - \|\gbf_i(\zbf)\|}{t} - \|\nabla \gbf_i(\xbf) \vbf\| } \\
    \le\ & \frac{\|\gbf_i(\zbf + t \vbf) - \gbf_i(\zbf) -t \nabla \gbf_i(\xbf) \vbf\|}{t} = \norm[3]{ \frac{1}{t} \int_0^t \nabla \gbf_i(\zbf + s \vbf) \vbf \dif s - \nabla \gbf_i(\xbf) \vbf} \\ 
    \le\ & \frac{1}{t} \int_0^t \norm[1]{\nabla \gbf_i (\zbf+s\vbf)\vbf - \nabla \gbf_i(\zbf)\vbf +  \nabla \gbf_i(\zbf)\vbf - \nabla \gbf_i(\xbf) \vbf } \dif s \\
    \le\ & \frac{1}{t}\int_0^t M(\|\vbf\|^2 s +\|\zbf - \xbf\|\|\vbf\|) \dif s 
    = M (\frac{t}{2} \|\vbf\|^2 + \| \zbf - \xbf \|\|\vbf\|) \to 0
\end{align*}
as $\zbf \to \xbf$ and $t \downarrow 0$, which implies that
\[
\limsup_{\zbf \rightarrow \xbf,\, t \downarrow 0} \frac{r_i(\zbf + t \vbf) - r_i(\zbf)}{t} \le \| \nabla \gbf_i(\xbf)\vbf \|.
\]
To show that the equality actually holds, we set $\zbf \equiv \xbf$ and let $t \ge 0$, then there are $\gbf_i(\zbf) = \gbf_i(\xbf) = \mathbf{0}$, which implies $r_i(\zbf) = \|\gbf_i(\zbf)\| = 0$, and
\begin{align*}
    &\abs[3]{\frac{r_i(\zbf + t \vbf) - r_i(\zbf)}{t} - \|\nabla \gbf_i(\xbf)\vbf\| } = \abs[3]{\frac{\|\gbf_i(\zbf + t \vbf)\| - \|\gbf_i(\zbf)\|}{t} - \|\nabla \gbf_i(\xbf) \vbf\| } \\
    \le\ & \frac{\|\gbf_i(\zbf + t \vbf) - \gbf_i(\zbf) -t \nabla \gbf_i(\xbf) \vbf\|}{t} \le M (\frac{t}{2} \|\vbf\|^2 + \| \zbf - \xbf \|\|\vbf\|) \to 0
\end{align*}
as $t \downarrow 0$, where the last inequality is due to the same deduction above. Hence we obtain 
\[
\limsup_{\zbf \rightarrow \xbf,\, t \downarrow 0} \frac{r_i(\zbf + t \vbf) - r_i(\zbf)}{t} = \| \nabla \gbf_i(\xbf)\vbf \|.\] Therefore, for any $\wbf \in \mathbb{R}^d$ satisfying $\|\Pi(\wbf; \Ccal(\nabla \gbf_i(\xbf)))\| \le 1$, we have
\begin{align*}
    & \langle \nabla \gbf_i(\xbf)^{\top} \wbf,  \vbf \rangle = \langle  \wbf, \nabla \gbf_i(\xbf) \vbf \rangle = \langle \Pi(\wbf; \Ccal(\nabla \gbf_i(\xbf))), \nabla \gbf_i(\xbf) \vbf \rangle \le \| \nabla \gbf_i(\xbf) \vbf \| = \limsup_{\zbf \rightarrow \xbf,\, t \downarrow 0} \frac{r_i(\zbf + t \vbf) - r_i(\zbf)}{t}
\end{align*}
where the second equality is due to $\nabla \gbf_i(\xbf) \vbf \in \Ccal(\nabla \gbf_i(\xbf))$.
On the other hand, for any $\wbf \in \mathbb{R}^d$ satisfying $\|\Pi(\wbf; \Ccal(\nabla \gbf_i(\xbf)))\| > 1$, there exists $\vbf \in \mathbb{R}^n$, such that $\nabla \gbf_i(\xbf) \vbf = \Pi(\wbf; \Ccal(\nabla \gbf_i(\xbf)))$ and
\[
\langle \nabla \gbf_i(\xbf)^{\top} \wbf,  \vbf \rangle = \langle \Pi(\wbf; \Ccal(\nabla \gbf_i(\xbf))), \nabla \gbf_i(\xbf) \vbf \rangle = \| \nabla \gbf_i(\xbf) \vbf \|^2 > \| \nabla \gbf_i(\xbf) \vbf \| = \limsup_{\zbf \rightarrow \xbf,\, t \downarrow 0} \frac{r_i(\zbf + t \vbf) - r_i(\zbf)}{t}.
\]
Therefore, by Definition \ref{def:clarke_subdiff}, we obtain the Clarke subdifferential $\partial r(\xbf)$ as in \eqref{eq:r_subdiff}.
\end{proof}
We immediately have the subdifferential $\partial \phi$ due to \eqref{eq:r_subdiff} and the differentiability of $f$:
\begin{equation}\label{eq:phi_subdiff}
\partial \phi(\xbf) = \partial r(\xbf) + \nabla f(\xbf).
\end{equation}
The following lemma also provides the Lipschitz constant of $\nabla \reps$.
\begin{lemma}\label{lem:r_Lip}
The gradient $\nabla \reps$ of $\reps$ defined in \eqref{eq:r_eps_def} is Lipschitz continuous with constant $\sqrt{m} L_g+\frac{M^2}{\varepsilon}$.
\end{lemma}
\begin{proof}
For any $\xbf_1,\xbf_2 \in \Xcal$, we first define $\ybf_1$ and $\ybf_2$ as follows,
\begin{align*}
\ybf_1 &=\argmax_{\ybf \in \Ycal}\ \langle \gbf(\xbf_1),\,\ybf \rangle -\frac{\varepsilon}{2} \| \ybf \|^2, \\
\ybf_2 &=\argmax_{\ybf \in \Ycal}\ \langle \gbf(\xbf_2),\,\ybf \rangle -\frac{\varepsilon}{2} \| \ybf \|^2,
\end{align*}
which are well defined since the maximization problems have unique solutions.
Due to the concavity of the problems above (in $\ybf$) and the optimality conditions of $\ybf_1$ and $\ybf_2$, we have
\begin{align*}
& \langle \gbf(\xbf_1)-\varepsilon \ybf_1,\,\ybf_2-\ybf_1\rangle\leq 0, \\
& \langle \gbf(\xbf_2)-\varepsilon \ybf_2,\,\ybf_1-\ybf_2\rangle\leq 0.
\end{align*}
Adding the two inequalities above yields
\begin{align*}
\langle \gbf(\xbf_1)-\gbf(\xbf_2)-\varepsilon \left(\ybf_1-\ybf_2\right),\,\ybf_2-\ybf_1\rangle\leq 0,
\end{align*}
which, together with the Cauchy-Schwarz inequality, implies
\begin{equation*}
    \varepsilon\, \|  \ybf_2-\ybf_1 \| ^2 \le \langle \gbf(\xbf_1)-\gbf(\xbf_2),\,\ybf_1-\ybf_2\rangle \le \|  \gbf(\xbf_1)-\gbf(\xbf_2) \|  \cdot  \|  \ybf_1-\ybf_2 \|.
\end{equation*}
Therefore, $\varepsilon\, \|  \ybf_1-\ybf_2 \| \le \|  \gbf(\xbf_1)-\gbf(\xbf_2) \|$.
Recall that $\nabla \reps(\xbf_j) = \nabla \gbf(\xbf_j)^{\top} \ybf_j$ for $j=1,2$. 
Therefore, we have
\begin{align*}
& \,\| \nabla r_{\varepsilon}(\xbf_1) - \nabla r_{\varepsilon}(\xbf_2) \| = \left \| \nabla \gbf(\xbf_1)^\top \ybf_1-\nabla \gbf(\xbf_2)^\top \ybf_2\right \| \\
= \ &  \left \| \left(\nabla \gbf(\xbf_1)^\top \ybf_1-\nabla \gbf(\xbf_2)^\top \ybf_1\right)+\left(\nabla \gbf(\xbf_2)^\top \ybf_1-\nabla \gbf(\xbf_2)^\top \ybf_2\right)\right \| \\
\le \ & \left \| \left(\nabla \gbf(\xbf_1)  -\nabla \gbf(\xbf_2) \right)^\top \ybf_1 \right \| +  \|  \nabla \gbf(\xbf_2)  \|  \left \|   \ybf_1-  \ybf_2 \right \| 
\\
\le \  & \left \|  \nabla \gbf(\xbf_1)  -\nabla \gbf(\xbf_2)  \right \|  \cdot \| \ybf_1   \| +  \frac{1}{\varepsilon}\cdot  \|  \nabla \gbf(\xbf_2)  \| \cdot  \|  \gbf(\xbf_1)-\gbf(\xbf_2) \| \\
\le \  & L_g \|  \xbf_1 - \xbf_2 \|  \cdot \| \ybf_1   \| +  \frac{M}{\varepsilon}\cdot  \|  \nabla \gbf(\xbf_2)  \| \cdot  \|  \xbf_1 - \xbf_2 \| ,
\end{align*}
where the last inequality is due to the $L_g$-Lipschitz continuity of $\nabla \gbf$ for the first term, and $\|\gbf(\xbf_1) - \gbf(\xbf_2)\|=\| \nabla \gbf(\tilde{\xbf})(\xbf_1-\xbf_2)\|$ for some $\tilde{\xbf}$ due to the mean value theorem and that $ \| \nabla \gbf(\tilde{\xbf}) \| \le \sup_{\xbf \in \Xcal} \| \nabla \gbf(\xbf) \| \leq M$ for the second term. 
Since $\max_{\ybf \in \Ycal} \|  \ybf  \| = \sqrt{m}$, we have
\[
\| \nabla r_{\varepsilon}(\xbf_1) - \nabla r_{\varepsilon}(\xbf_2) \| \le \left \| \nabla \gbf(\xbf_1)^\top \ybf_1-\nabla \gbf(\xbf_2)^\top \ybf_2\right \| \leq\del[2]{\sqrt{m} L_g+\frac{M^2}{\varepsilon} }\,  \|  \xbf_1-\xbf_2 \|,
\]
which completes the proof.
\end{proof}

Now we return to Algorithm \ref{alg:lda}. We first consider its behavior if a constant $\varepsilon>0$ is used, i.e., an iterative scheme that only executes its Lines 3--6.

\begin{lemma}\label{lem:inner}
Let $\varepsilon, \eta>0$, $\delta_1 \ge \delta_2 > 1$ and $\xbf_0 \in \Xcal$ be arbitrary. Suppose $\{\xbf_{k}\}$ is the sequence generated by repeating Lines 3--6 of Algorithm \ref{alg:lda} with $\varepsilon_k = \varepsilon$ and step sizes $\frac{1}{\delta_1 L_{\varepsilon}} \le \alpha_k \le \frac{1}{\delta_2 L_{\varepsilon}}$ for all $k\ge 0$, where $L_{\varepsilon} = L_f + \sqrt{m} L_g + \frac{M^2}{\varepsilon}$, and $\phi^*:=\min_{\xbf \in \Xcal} \phi(\xbf)$. Then the following statements hold:
\begin{enumerate}
    \item $\| \nabla \phi_{\varepsilon}(\xbf_k) \| \to 0$ as $k\to \infty$.
    \item $\min\{ k \in \mathbb{N} \ \vert \ \| \nabla \phi_{\varepsilon}(\xbf_{k+1}) \| \le \eta \} \le \frac{\delta_1 \delta_2 L_{\varepsilon}(2\phi_{\varepsilon}(\xbf_0) - 2 \phi^* + m\varepsilon)}{(\delta_2 - 1 ) \eta^2}$.
\end{enumerate}
\end{lemma}
\begin{proof}
1. Due to the optimality condition of $\vbf_{k+1}$ in \eqref{eq:v}, we have 
\begin{equation}\label{eq:v_opt}
    \langle\nabla \phi_{\varepsilon}(\xbf_{k}), \vbf_{k+1}-\xbf_{k}\rangle+\frac{1}{2\alpha_k} \|  \vbf_{k+1}-\xbf_{k} \| ^2 \leq 0.
\end{equation}
In addition, $\nabla \phi_{\varepsilon}$ is $L_\varepsilon$-Lipschitz continuous due to Lemma \ref{lem:r_Lip}, which implies that
\begin{equation}\label{lipF}
  \phi_{\varepsilon}(\vbf_{k+1}) \leq \phi_{\varepsilon}(\xbf_{k}) + \langle\nabla \phi_{\varepsilon}(\xbf_{k}), \vbf_{k+1}-\xbf_{k}\rangle+\frac{L_{\varepsilon}}{2} \|  \vbf_{k+1}-\xbf_{k} \| ^2.
\end{equation}
Combining $\eqref{eq:v_opt}$, $\eqref{lipF}$ and $\vbf_{k+1} = \xk - \alpha_k \nabla \phi_{\varepsilon}(\xbf_k)$ in \eqref{eq:v_closed} yields
\begin{equation}\label{eq:diff}
\phi_{\varepsilon}(\vbf_{k+1})-\phi_{\varepsilon}(\xbf_{k})\leq - \del[2]{\frac{1}{2\alpha_k} -\frac{L_{\varepsilon}}{2}} \|  \vbf_{k+1}-\xbf_{k} \|^2 = - \frac{\alpha_k(1- \alpha_k L_{\varepsilon})}{2} \|  \nabla \phi_{\varepsilon} (\xbf_k)\|^2 \le 0,
\end{equation}
where we used the fact that $\alpha_k L_{\varepsilon} \le \frac{1}{\delta_2} < 1$ to obtain the last inequality.
According to the selection rule \eqref{eq:x_select}, if $\phi_{\varepsilon}(\ubf_{k+1})\leq \phi_{\varepsilon}(\vbf_{k+1})$, then $\xbf_{k+1}=\ubf_{k+1}$, and $\phi_{\varepsilon}(\xbf_{k+1})=\phi_{\varepsilon}(\ubf_{k+1})\leq \phi_{\varepsilon}(\vbf_{k+1})$;
If $\phi_{\varepsilon}(\vbf_{k+1})< \phi_{\varepsilon}(\ubf_{k+1})$, then $\xbf_{k+1}=\vbf_{k+1}$, and $\phi_{\varepsilon}(\xbf_{k+1})=\phi_{\varepsilon}(\vbf_{k+1})$. 
Therefore, in either case, \eqref{eq:diff} implies $\phi_{\varepsilon}(\xkp) - \phi_{\varepsilon}(\xk) \le \phi_{\varepsilon}(\vkp) - \phi_{\varepsilon}(\xk) \le  0 $, and hence
\begin{equation}
\label{eq:recursive}
\phi_{\varepsilon}(\xbf_{k+1}) \le \phi_{\varepsilon}(\vbf_{k+1}) \le \phi_{\varepsilon}(\xbf_{k}) \le \cdots \le \phi_{\varepsilon}(\xbf_{0}),
\end{equation}
for all $k\ge 0$. 
Moreover, rearranging \eqref{eq:diff} and recalling that $\frac{1}{\delta_1 L_{\varepsilon}} \le \alpha_k  \le \frac{1}{\delta_2 L_{\varepsilon}}$ yield
\begin{equation}\label{eq:diff2}
   \frac{\delta_2 - 1}{2\delta_1 \delta_2 L_{\varepsilon}} \|  \nabla \phi_{\varepsilon} (\xbf_k)\|^2 \le  \frac{\alpha_k(1- \alpha_k L_{\varepsilon})}{2} \|  \nabla \phi_{\varepsilon} (\xbf_k)\|^2 \le \phi_{\varepsilon}(\xbf_{k})-\phi_{\varepsilon}(\vbf_{k+1}) \le \phi_{\varepsilon}(\xbf_{k})-\phi_{\varepsilon}(\xbf_{k+1}).
\end{equation}
Summing up \eqref{eq:diff2} for $k=0,\dots,K$ and using the fact that $\phi_{\varepsilon}(\xbf) \ge \phi(\xbf)-\frac{m\varepsilon}{2} \ge \phi^* - \frac{m\varepsilon}{2}$ for every $\xbf \in \Xcal$, we know that
\begin{equation}
    \sum_{k=0}^{K} \|  \nabla \phi_{\varepsilon} (\xbf_k) \| ^2 \leq \frac{2\delta_1 \delta_2 L_{\varepsilon}(\phi_{\varepsilon}(\xbf_0) - \phi_{\varepsilon}(\xbf_{K+1}))}{\delta_2 - 1} \le \frac{\delta_1 \delta_2 L_{\varepsilon}(2\phi_{\varepsilon}(\xbf_0) - 2\phi^* + m \varepsilon)}{\delta_2 - 1}.
\end{equation}
Note that the right hand side is a finite constant, and hence by letting $K\to\infty$ we know that $\| \nabla \phi_{\varepsilon} (\xbf_k) \| \to 0$, which proves the first statement.

2. Denote $\kappa := \min\{ k \in \mathbb{N} \ \vert \ \| \nabla \phi_{\varepsilon}(\xbf_{k+1}) \| < \eta \}$, then we know that $\| \nabla \phi_{\varepsilon} (\xbf_{k+1}) \| \ge \eta$ for all $k \le \kappa-1$.
Hence we have 
\[
\kappa \eta^2 \le \sum_{k=0}^{\kappa-1} \|  \nabla \phi_{\varepsilon} (\xbf_{k+1}) \|^2 =\sum_{k=1}^{\kappa} \|  \nabla \phi_{\varepsilon} (\xbf_k) \| ^2 \leq \frac{\delta_1 \delta_2 L_{\varepsilon}(2\phi_{\varepsilon}(\xbf_0) - 2\phi^* + m \varepsilon)}{\delta_2 - 1},
\]
which implies the second statement.
\end{proof}

Now we consider the complete version of Algorithm \ref{alg:lda}. The first result we have is on the monotonicity of $\phi_{\varepsilon_{k}}(\xbf_{k}) + \frac{m \varepsilon_{k}}{2}$ in $k$.
\begin{lemma}
\label{lem:phi_decay}
Suppose that the sequence $\{\xbf_k\}$ is generated by Algorithm \ref{alg:lda} with $\frac{1}{\delta_1 L_{\varepsilon_k}} \le \alpha_k \le \frac{1}{\delta_2 L_{\varepsilon_k}}$ and any initial $\xbf_0$. Then for any $k\ge 0$ there is 
\begin{equation}\label{eq:phi_decay}
    \phi_{\varepsilon_{k+1}}(\xbf_{k+1}) + \frac{m \varepsilon_{k+1}}{2} \le \phi_{\varepsilon_{k}}(\xbf_{k+1}) + \frac{m \varepsilon_{k}}{2} \le \phi_{\varepsilon_{k}}(\xbf_{k}) + \frac{m \varepsilon_{k}}{2}.
\end{equation}
\end{lemma}

\begin{proof}
Due to \eqref{eq:recursive}, the second inequality holds immediately. So we focus on the first inequality.
For any $\varepsilon>0$ and $\xbf$, denote
\begin{equation}\label{eq:rei}
r_{\varepsilon, i}(\xbf) := \begin{cases}
\frac{1}{2 \varepsilon} \| \gbf_i(\xbf)\|, & \mbox{if} \ \|\gbf_i(\xbf)\| \le \varepsilon, \\
\| \gbf_i(\xbf)\| - \frac{\varepsilon}{2}, & \mbox{if} \ \|\gbf_i(\xbf)\| > \varepsilon .
\end{cases}
\end{equation}
Then it is clear that $\phi_{\varepsilon}(\xbf) = \sum_{i=1}^m r_{\varepsilon,i}(\xbf) + f(\xbf)$.
To prove the first inequality, it suffices to show that
\begin{equation}
    \label{eq:r_decay}
    r_{\varepsilon_{k+1},i}(\xkp) + \frac{\varepsilon_{k+1}}{2} \le r_{\epsk,i}(\xkp) + \frac{\epsk}{2}.
\end{equation}
If $\varepsilon_{k+1} = \varepsilon_{k}$, then the two quantities above are identical and the first inequality holds. Now suppose $\varepsilon_{k+1} = \gamma \varepsilon_{k} < \varepsilon_k$. %
We then consider the relation between $\| \gbf_i(\xbf_{k+1})\|$, $\varepsilon_{k+1}$ and $\varepsilon_k$ in three cases:
(i) If $\| \gbf_i(\xkp) \| > \varepsilon_{k} > \varepsilon_{k+1}$, then by the definition in \eqref{eq:rei}, there is
\[
r_{\varepsilon_{k+1},i}(\xkp) + \frac{\varepsilon_{k+1}}{2} = \| \gbf_i(\xkp) \| = r_{\varepsilon_{k},i}(\xkp) + \frac{\varepsilon_{k}}{2}.
\]
(ii) If $\varepsilon_{k} \ge \| \gbf_i(\xkp) \| > \varepsilon_{k+1}$, then \eqref{eq:rei} implies
\[
r_{\varepsilon_{k+1},i}(\xkp) + \frac{\varepsilon_{k+1}}{2} = \frac{\| \gi(\xkp) \|^2}{2\epskp} + \frac{\epskp}{2} \le \frac{\|\gi(\xkp)\|}{2} + \frac{\|\gi(\xkp)\|}{2} = r_{\epsk,i}(\xkp) + \frac{\epsk}{2}.
\]
(iii) If $\varepsilon_{k}  > \varepsilon_{k+1} \ge \| \gbf_i(\xkp) \|$, then we know that $\frac{\|\gi(\xkp)\|^2}{2\varepsilon} + \frac{\varepsilon}{2}$---as a function of $\varepsilon$---is non-decreasing for all $\varepsilon \ge \| \gi(\xkp)\|^2 $, which implies \eqref{eq:r_decay}.
Therefore, in either of the three cases, \eqref{eq:r_decay} holds and hence
\[
r_{\varepsilon_{k+1}}(\xkp) + \frac{m\varepsilon_{k+1}}{2} = \sum_{i=1}^m \del[2]{ r_{\varepsilon_{k+1},i}(\xkp) + \frac{\varepsilon_{k+1}}{2}} \le \sum_{i=1}^m \del[2]{ r_{\epsk,i}(\xkp) + \frac{\epsk}{2} } = r_{\varepsilon_{k}}(\xkp) + \frac{m\varepsilon_{k}}{2},
\]
which implies the first inequality of \eqref{eq:phi_decay}.
\end{proof}

Now we are ready to prove the iteration complexity of Algorithm \ref{alg:lda} for any $\etol>0$. Note that Lemma \ref{lem:inner} implies that the reduction criterion in Line 7 of Algorithm \ref{alg:lda} must be satisfied within finitely many iterations since it was met last time, and hence $\varepsilon_k$ will eventually be small enough to satisfy Line 8 and terminate the algorithm. 
Let $k_{l} $ be the counter of iteration when the criterion in Line 7 of Algorithm \ref{alg:lda} is met for the $l$-th time (we set $k_{0}=-1$), then we can partition the iteration counters $k=0,1,2,\dots,$ into segments accordingly, such that $\varepsilon_k = \varepsilon_{k_{l} +1} = \varepsilon_0 \gamma^l$ for $k=k_{l} +1,\dots,k_{l+1} $ in the $l$-th segment.
From Lemma \ref{lem:inner}, we can bound the length of each segment and hence the total iteration number which is the sum of these lengths. These results are given in the following theorem.

\begin{theorem} \label{thm:eta_shrinking}
Suppose that $\{\xbf_{k}\}$ is the sequence generated by Algorithm \ref{alg:lda} with any initial $\xbf_0$ and step size $\frac{1}{\delta_1 \Lepsk} \le \alpha_k \le \frac{1}{\delta_2 \Lepsk}$. Then the following statements hold:
\begin{enumerate}
\item The number of iterations, $k_{l+1} - k_{l}$, for the $l$-th segment is bounded by
\begin{equation}
\label{eq:inner_bound}
k_{l+1} - k_{l} \le c_1 \gamma^{-2l} + c_2 \gamma^{-3l},
\end{equation}
where the constants $c_1$ and $c_2$ are defined by
\begin{equation}\label{eq:c_def}
c_1 = \frac{\delta_1 \delta_2 (L_f + \sqrt{m}L_g)(2\phi(\xbf_0) - 2\phi^* + m\varepsilon_0)}{(\delta_1 - 1) \sigma^2 \varepsilon_0^2 \gamma^2},\quad
c_2 = \frac{\delta_1 \delta_2 M^2 (2\phi(\xbf_0) - 2\phi^* + m\varepsilon_0)}{(\delta_1 - 1) \sigma^2 \varepsilon_0^3 \gamma^3}.
\end{equation}
\item The total number of iterations for Algorithm \ref{alg:lda} to terminate with $\etol>0$ is bounded by
\begin{equation}
\label{eq:total_iter}
\frac{c_1 \sigma^2 \varepsilon_0^2}{ 1 - \gamma^2} \etol^{-2} + \frac{c_2 \sigma^3 \varepsilon_0^3}{1-\gamma^3} \etol^{-3} - \frac{c_1 \gamma^2 + c_2 \gamma^3 - (c_1 + c_2) \gamma^5}{(1-\gamma^2)(1 - \gamma^3)} = O(\etol^{-3}).
\end{equation}
\end{enumerate}
\end{theorem}

\begin{proof} 
1. Due to Lemma \ref{lem:phi_decay}, we know that, for all $k\ge0$, there is
\begin{equation}\label{eq:phi_eps_bound}
\phi_{\epskp} (\xkp) + \frac{m \epskp}{2} \le \phi_{\epsk}(\xk) + \frac{m\epsk}{2} \le \cdots \le \phi_{\varepsilon_0}(\xbf_0) + \frac{m \varepsilon_0}{2} \le \phi(\xbf_0) + \frac{m \varepsilon_0}{2}
\end{equation}
where we used the fact that $\phi_{\varepsilon}(\xbf) \le \phi(\xbf)$ for all $\varepsilon>0$ and $\xbf \in \Xcal$ in the last inequality.
Therefore $\klp - \kl$ satisfies the bound in Lemma \ref{lem:inner} (Statement 2) with $\varepsilon = \varepsilon_{\kl+1} = \varepsilon_0 \gamma^l$, $\eta = \sigma \gamma \varepsilon_{\kl+1} = \sigma \varepsilon_0 \gamma^{l+1}$ and initial $\xbf_{\kl+1}$. Namely, there is
\begin{align*}
\klp - \kl 
\le\  & \frac{2\delta_1 \delta_2 (L_f + \sqrt{m} L_g + \frac{M^2}{\varepsilon}) (\phi_{\varepsilon}(\xbf_{\kl + 1}) - \phi^* + \frac{m\varepsilon}{2})}{(\delta_1 -1) \eta^2} \\
\le\  & \frac{\delta_1 \delta_2 (L_f + \sqrt{m} L_g) (2\phi(\xbf_{0}) - 2\phi^* + m\varepsilon_0)}{(\delta_1 -1) \eta^2} + \frac{\delta_1 \delta_2 M^2 (2\phi(\xbf_{0}) - 2\phi^* + m\varepsilon_{0})}{(\delta_1 -1) \varepsilon \eta^2}\\
=\ & c_1 \gamma^{-2l} + c_2 \gamma^{-3l},
\end{align*}
where we used \eqref{eq:phi_eps_bound} to obtain $\phi_{\varepsilon}(\xbf_{\kl+1}) + \frac{m \varepsilon}{2 } \le \phi(\xbf_0) + \frac{m \varepsilon_0}{2}$ for $\varepsilon = \varepsilon_{\kl+1} $ in the second inequality and the definitions of $c_1$ and $c_2$ in \eqref{eq:c_def} to obtain the last equality.

2. Let $\ell$ be the number of times the reduction criterion in Line 7 of Algorithm \ref{alg:lda} is satisfied before it is terminated by Line 8. Then $ \sigma \varepsilon_0 \gamma^{\ell-1} \ge \etol$. Hence we have  $\ell - 1\le \log_{\gamma}^{(\sigma\varepsilon_0)^{-1}\etol}$, which implies that the total number of iterations for Algorithm \ref{alg:lda} to terminate with $\etol$ is 
\begin{equation*}
\sum_{l=0}^{\ell - 1} (\klp - \kl) \le \sum_{l=0}^{\ell - 1} (c_1 \gamma^{-2l} + c_2 \gamma^{-3l}) \le
\frac{c_1(\gamma^{-2(\ell-1)} - \gamma^2)}{1-\gamma^2} + \frac{c_2(\gamma^{-3(\ell-1)}- \gamma^3)}{1-\gamma^3} 
\end{equation*}
and readily reduces to \eqref{eq:total_iter}. This completes the proof.
\end{proof}

\newblue{Theorem \ref{thm:eta_shrinking} provides the upper bound of the iteration complexity of Algorithm \ref{alg:lda} to reach an $\epsilon_{\text{tol}}$ accurate solution of the nonsmooth nonconvex problem \eqref{eq:loa} for any user chosen $\epsilon > 0$. We also remark that the complexity analysis in Theorem \ref{thm:eta_shrinking} has taken into account the fact that the Lipschitz constant of $\nabla \phi_{\varepsilon_k}$ gradually increases as $\varepsilon_k$ decreases. Moreover, when the termination condition in Step 8 of Algorithm \ref{alg:lda} is met, the smoothing parameter $\varepsilon$ satisfies $\varepsilon = \varepsilon_k <\epsilon_{\text{tol}}/\sigma$. Note that the bound \eqref{eq:r_eps_bound} controls the distance between $\reps$ and the original nonsmooth $r$ and thus that between $\phi_{\varepsilon}$ and $\phi$. In addition, there is $\|\nabla \phi_{\varepsilon}\| < \sigma \gamma \varepsilon_k < \gamma \epsilon_{\text{tol}}$. Therefore, one can choose $\epsilon_{\text{tol}}$, $\sigma$ and $\gamma$ to achieve the desired accuracy.}

If we set $\etol=0$ and $K=\infty$ in Algorithm \ref{alg:lda}, 
then LDA will generate an infinite sequence $\{\xk\}$. We focus on the subsequence $\{\xbf_{\kl+1}\}$ which selects the iterates when the reduction criterion in Line 7  is satisfied for $k = \kl$ and $\varepsilon_k$ is reduced. Then we can show that every accumulation point of this subsequence is a Clarke stationary point, as shown in the following theorem.
%

\begin{theorem}
Suppose that $\{\xbf_{k}\}$ is the sequence generated by Algorithm \ref{alg:lda} with any initial $\xbf_0$ and step size $\frac{1}{\delta_1 \Lepsk} \le \alpha_k \le \frac{1}{\delta_2 \Lepsk}$, $\etol=0$ and $K=\infty$. Let $\{\xbf_{\kl+1}\}$ be the subsequence where the reduction criterion Line 7 of Algorithm \ref{alg:lda} is met for $k=k_l$ and $l=1,2,\dots$. Then the following statements hold:
\begin{enumerate}
\item
$\{\xbf_{\kl + 1}\}$ has at least one accumulation point.
\item
Every accumulation point of $\{\xbf_{\kl + 1}\}$ is a Clarke stationary point of \eqref{eq:phi}.
\end{enumerate}
\end{theorem}

\begin{proof}
1. Due to Lemma \ref{lem:phi_decay} and $\phi(\xbf) \le \phi_{\varepsilon}(\xbf) + \frac{m \varepsilon}{2}$ for all $\varepsilon>0$ and $\xbf \in \Xcal$, we know that
\begin{equation*}
\phi(\xk) \le \phi_{\epsk}(\xk) + \frac{m\epsk}{2} \le \cdots \le \phi_{\varepsilon_0}(\xbf_0) + \frac{m\varepsilon_0}{2} < \infty.
\end{equation*}
Since $\phi$ is coercive, we know that $\{\xk\}$ is bounded. Hence $\{ \xbf_{\kl+1}\}$ is also bounded and has at least one accumulation point.

2. Note that $\xbf_{\kl+1}$ satisfies the reduction criterion in Line 7 of Algorithm \ref{alg:lda}, i.e., $\| \nabla \phi_{\varepsilon_{\kl}} (\xbf_{\kl+1}) \| \le \sigma \gamma \varepsilon_{\kl} = \sigma \varepsilon_0 \gamma^{l+1} \to 0$ as $l \to \infty$.
For notation simplicity, we let $\{\xjp\}$ denote any convergent subsequence of $\{\xbf_{\kl +1}\}$ and $\varepsilon_j$ the corresponding $\epsk$ used in the iteration to generate $\xjp$. Then there exists $\xhat \in \Xcal$ such that $\xjp \to \xhat$, $\epsj \to 0$, and $\nabla \phi_{\epsj}(\xjp) \to 0$ as $j\to \infty$.

Recall the Clarke subdifferential of $\phi$ at $\xhat$ is given by \eqref{eq:phi_subdiff}:
\begin{equation}\label{eq:d_phi_xhat}
\partial \phi(\xhat) = \cbr[2]{\sum_{i \in I_0} \nabla \gi(\xhat)^{\top} \wbf_i + \sum_{ i \in I_1} \nabla \gi(\xhat)^{\top} \frac{\gi(\xhat)}{\| \gi(\xhat) \|} + \nabla f(\xhat) \ \bigg\vert \ \| \Pi(\wbf_i; \Ccal(\nabla \gi(\xhat))) \le 1,\ \forall\, i\in I_0},
\end{equation}
where $I_0 = \{i\in[m]\ \vert \ \|\gi(\xhat)\| = 0 \}$ and $I_1 = [m] \setminus I_0$.
Then we know that there exists $J$ sufficiently large, such that 
\[
\epsj < \frac{1}{2}\min \{ \|\gi(\xhat)\| \ \vert \ i\in I_1\} \le \frac{1}{2} \|\gi(\xhat)\| \le \| \gi(\xjp) \|, \quad \forall\, j\ge J,\quad \forall\, i\in I_1,
\]
where we used the facts that $\min \{ \|\gi(\xhat)\| \ \vert \ i\in I_1\}>0$ and $\epsj \to 0$ in the first inequality, and $\xjp \to \xhat$ and the continuity of $\gi$ for all $i$ in the last inequality.
Furthermore, we denote
\begin{equation*}
\sji := \begin{cases}
\frac{\gi(\xbf_{j+1})}{\epsj}, & \mbox{if}\ \|\gi(\xjp)\| \le \epsj, \\
\frac{\gi(\xbf_{j+1})}{\| \gi(\xjp) \|}, & \mbox{if}\ \|\gi(\xjp)\| > \epsj.
\end{cases}
\end{equation*}
Then we have
\begin{equation}\label{eq:d_phi_epsj}
\nabla \phi_{\epsj}(\xjp) = \sum_{i \in I_0} \nabla \gi(\xjp)^{\top} \sji + \sum_{ i \in I_1} \nabla \gi(\xjp)^{\top} \frac{\gi(\xjp)}{\| \gi(\xjp) \|} + \nabla f(\xjp).
\end{equation}
Comparing \eqref{eq:d_phi_xhat} and \eqref{eq:d_phi_epsj}, we can see that the last two terms on the right hand side of \eqref{eq:d_phi_epsj} converge to those of \eqref{eq:d_phi_xhat}, respectively, due to the facts that $\xjp \to \xhat$ and the the continuity of $\gi,\nabla \gi, \nabla f$. Moreover, noting that $\|\Pi(\sji; \Ccal(\nabla \gi(\xhat)))\| \le \| \sji \| \le 1$, we can see that the first term on the right hand side of \eqref{eq:d_phi_epsj} also converges to the set formed by the first term of \eqref{eq:d_phi_xhat} due to the continuity of $\gi$ and $\nabla \gi$. Hence we know that
\[
\dist( \nabla \phi_{\epsj}(\xjp), \partial \phi(\xhat)) \to 0,
\]
as $j \to 0$. Since $\nabla \phi_{\epsj}(\xjp) \to 0$ and $\partial \phi(\xhat)$ is closed, we conclude that $0 \in \partial \phi(\xhat)$.
\end{proof}

The analysis above shows the convergence properties and the iteration complexity of the proposed LDA. \newblue{In particular, any accumulation point of the specified subsequence of LDA is guaranteed to be a Clarke stationary point.} It is worth pointing out that, unlike most works in the literature, our convergence guarantee and the iteration complexity do not require KL property.

\section{Numerical Experiments}
\label{sec:experiment}

\subsection{Network architecture and parameter setting}
\label{subsec:network}
Throughout our experiments, we parameterize $\gbf$ in \eqref{eq:r} as a simple 4-layer convolutional neural network with componentwise activation function $a$ and no bias as follows: 
\begin{equation}\label{eq:g_net}
\begin{cases}
\mbox{For any $\xbf$, compute}\ \gbf(\xbf) = \hbf_4, \\
\mbox{where}\ \hbf_0 = \xbf,\ \mbox{and} \\
\hbf_{l} = a(\Wbf_{l-1} \hbf_{l-1}),\quad l=1,2,3,4,\\
\end{cases}
\quad
\mbox{and}
\quad
a (x) = 
\begin{cases}
0, & \mbox{if} \ x \leq -\delta, \\
\frac{1}{4\delta} x^2 + \frac{1}{2} x + \frac{\delta}{4}, & \mbox{if} \ -\delta < x < \delta, \\
x, & \mbox{if} \ x \geq \delta,
\end{cases}
\end{equation}
where $\delta = 0.01$ in our experiment. In \eqref{eq:g_net}, $\Wbf_l$ represents the convolution in the $l$-th layer. We set the kernel size to $3\times 3 \times d$ for all layers, where $d=32$ is the depth of the convolution kernel. In our experiments, we set stride to 1, and use zero-padding to preserve image size.
Then $\Wbf_0$ can be interpreted as a $dn\times n$ matrix with $3^2 \times 32$ learnable parameters and $\Wbf_{l}$ as $dn\times dn$ for $l=1,2, 3$ each with $3^2 \times 32^2$ learnable parameters.
In this case $m=n$ is the number of pixels in the image.
Note that $\gbf$ satisfies Assumption (A2) due to the boundedness of $a'$ and the fixed $\Wbf_l$ once learned.
The regularization is $r(\xbf) = \|\gbf(\xbf)\|_{2,1}$ as in \eqref{eq:r}, and $\reps$ and $\nabla \reps$ are given in \eqref{eq:r_eps} and \eqref{eq:d_r_eps}, respectively.
%

%
During training, we prescribe the iteration number $K=15$ for Algorithm \ref{alg:lda} which seems to reach a good compromise between network depth and performance in practice.
We adopt a warm start strategy by first training LDA with $K=3$ for 500 epochs, and then add 2 more phases and train the network for another 200 epochs, and so on, until we finish with $K=15$. 
The step sizes $\alpha_k$ and $\tau_k$ are also to be learned and allowed to vary across different phases. The threshold $\epsk$ is updated according to Algorithm \ref{alg:lda}, where the starting $\varepsilon_{0}$ is to be learned. 
%
We let $\theta$ denote the set of trainable parameters of LOA in Algorithm \ref{alg:lda}, including the convolutions $\{\Wbf_l\}_{l=0}^3$, the step sizes $\{\alpha_k, \tau_k\}_{k = 0}^K$ and the starting $\varepsilon_0$. 

Given $N$ training data pairs $\{(\bbf^{(s)}, \hat{\xbf}^{(s)} )\}_{s=1}^{N}$, where each $\hat{\xbf}^{(s)}$ is the ground truth data and $\bbf^{(s)}$ is the measurement of $\hat{\xbf}^{(s)}$, we solve $\theta$ by minimizing the loss function in \eqref{eq:loa} using the Adam Optimizer with learning rate $10^{-4}$ and $\beta_1=0.9$, $\beta_2=0.999$ and Xavier Initializer implemented in TensorFlow \cite{abadi2016tensorflow}.
All the experiments are performed on a desktop computer with Intel i7-6700K CPU at 3.40 GHz, 16 GB of memory, and an Nvidia GTX-1080Ti GPU of 11GB graphics card memory.

\subsection{Experimental results on image reconstruction}
\begin{figure}[t]
\centering 
\includegraphics[width=1.0\textwidth]{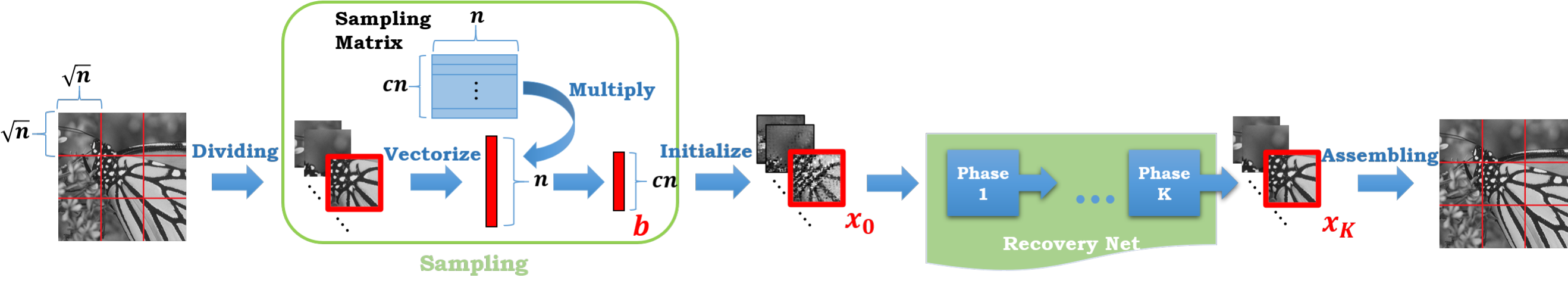}
\includegraphics[width=0.8\textwidth]{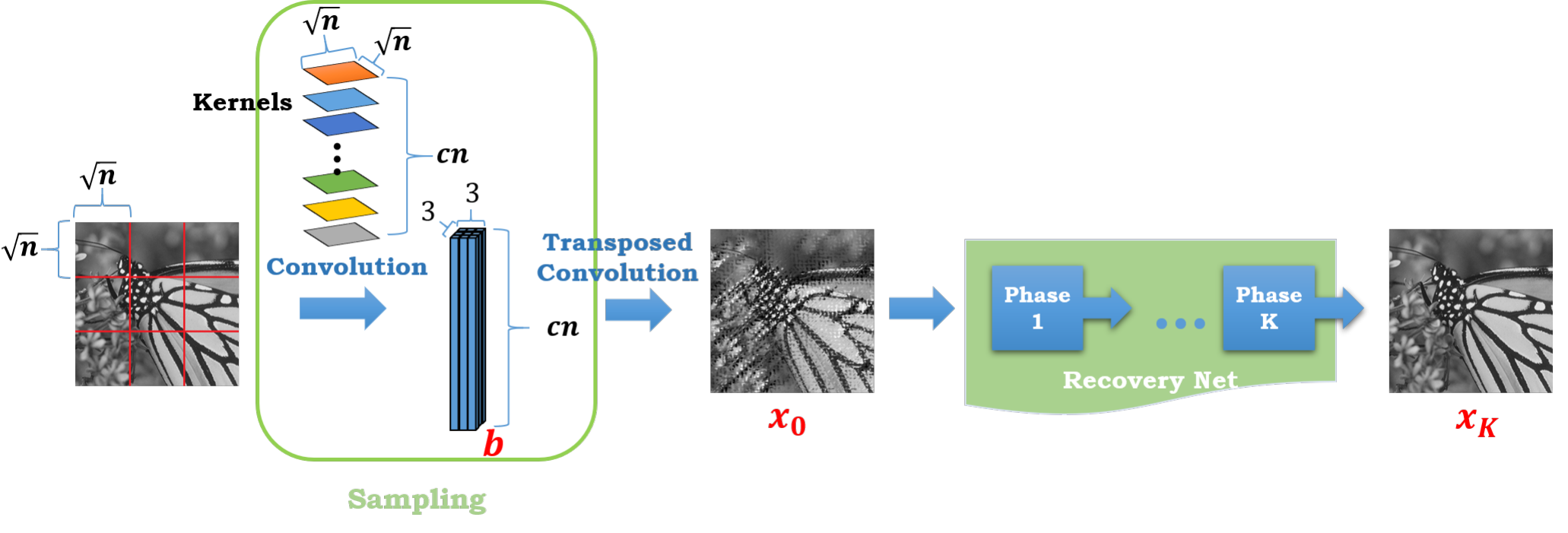}
\includegraphics[width=0.85\textwidth]{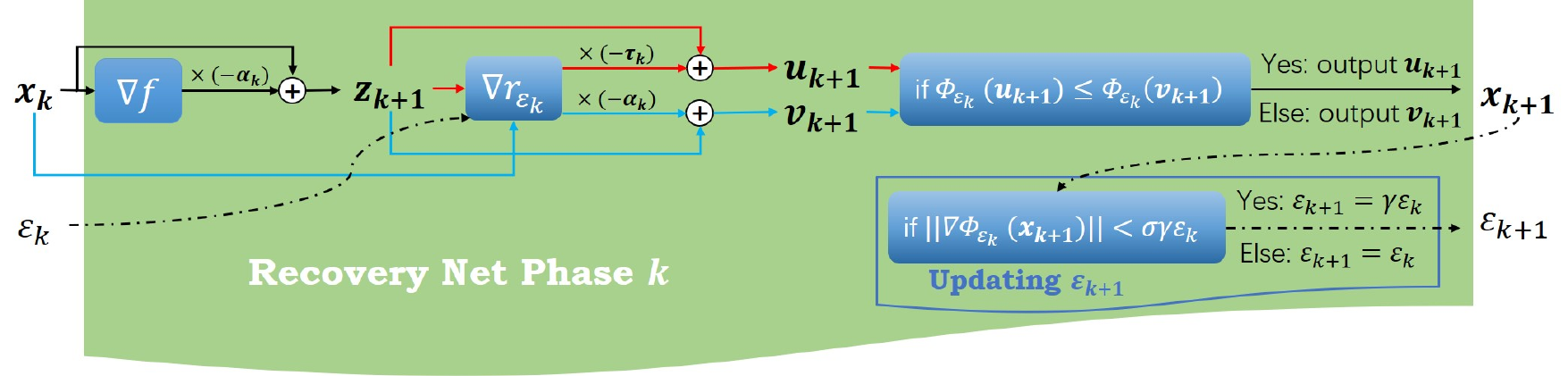}
\caption{The flowchart of the block compressed sensing natural image reconstruction. An image is partitioned into patches of size $n$, each of which, denoted by $\xhat$, is compressed by the sampling matrix $\Abf$ into data $\bbf= \Abf\hat{\xbf} \in \mathbb{R}^{cn}$. Top: The compressed data $\bbf$ is obtained using a prescribed sampling matrix $\Abf$ and is mapped to $\xbf_0$ as the initial value of the $K$-phase LDA reconstruction network; Middle: The sampling matrix $\Abf$ is jointly learned with the network parameters by appending $\bbf= \Abf\hat{\xbf} \in \mathbb{R}^{cn}$ as a linear layer before the LDA; Bottom: The detailed illustration of $k$th-phase Recovery Net.
}
\label{fig:flowchart}
\end{figure}
\subsubsection{Reconstruction on natural image compressed sensing}
\label{subsubsec:cs}
We first consider the natural image block compressed sensing (block CS) problem \cite{gan2007block} to recover images (image patches) from compressed data. In block CS, an image is partitioned into small blocks of size $n=33\times 33$, each of which (treated as a vector $\xhat\in \mathbb{R}^{n}$) is left-multiplied by a prescribed sensing matrix $\Abf \in \mathbb{R}^{cn \times n}$ to obtain the compressed data $\bbf= \Abf \xhat \in \mathbb{R}^{cn}$, where $c \in (0,1)$ is the compression ratio (CS ratio) \cite{dinh2013measurement,fowler2012block,gan2007block}. 
The flowchart of this process, including the compressed sensing part using the prescribed sampling matrix $\Abf$ and the reconstruction by a $K$-phase LDA network, is shown in the top panel of Figure \ref{fig:flowchart}.

We test the proposed Algorithm \ref{alg:lda} LDA on \textit{91 Images} for training and \textit{Set11} for testing \cite{kulkarni2016reconnet}. 
The training set $\mathcal{D}$ consists of $N = 88,912$ pairs of the form $(\bbf,\xhat) \in \mathbb{R}^{cn} \times \mathbb{R}^n$, where $\xhat$ is randomly cropped from the images.
The experiments on three different CS ratios $c=10\%,25\%,50\%$ are performed. 
The matrix $\Abf$ is set to a random Gaussian matrix whose rows are orthogonalized and the initial $\xbf_0$ is set to be $\xbf_0 = \mathbf{Q} \bbf$, where $\mathbf{Q} = \hat{\Xbf}\Bbf^{\top}(\Bbf\Bbf^{\top})^{-1}$ and $\hat{\Xbf} = [\hat{\xbf}^{(1)}, ... , \hat{\xbf}^{(N)}]$, $\Bbf = [\bbf^{(1)}, ... , \bbf^{(N)}]$, which follows \cite{Zhang2018ISTANetIO}.
We follow the same criterion when generating the testing data pairs from \textit{Set11}.
All the testing results are evaluated on the average Peak Signal-to-Noise Ratio (PSNR) of the reconstruction quality.
We compare with two classical image reconstruction methods, i.e., 
TVAL3 \cite{li2013efficient} and D-AMP \cite{metzler2016denoising}, and five state-of-the-art methods based on deep learning approaches, i.e., IRCNN \cite{zhang2017learning}, ReconNet \cite{kulkarni2016reconnet}, DR$^2$-Net \cite{yao2019dr2}, ISTA-Net$^+$ \cite{Zhang2018ISTANetIO} and DPA-Net \cite{sun2020dual}.
The comparison results on \textit{Set11} \cite{kulkarni2016reconnet} are listed in Table \ref{tab:cs-result},where the results of the first four methods and ISTA-Net$^+$ are quoted from \cite{Zhang2018ISTANetIO}.
The number of learnable parameters in the networks are also shown in the last column of Table \ref{tab:cs-result}. In general, a network has higher capacity and yields lower reconstruction error with more parameters (e.g., ISTA-Net$^+$ with varying parameters across different phases yields higher PSNR than that with parameters shared by all phases), but may also suffer the issue of parameter overfitting. As LDA uses the same set of network parameters in all phases, except the step size which is different in each phase but is only a scalar to be learned, it requires much fewer parameters than IRCNN, DR$^2$-Net, ISTA-Net$^+$ and DPA-Net. From Table \ref{tab:cs-result}, we can see that the proposed LDA obtained higher accuracy in reconstruction while using relatively small amount of network parameters.
\begin{table}[H]
\caption{Average PSNR (dB) of reconstructions obtained by the compared methods and the proposed LDA on \textit{Set11} dataset with CS ratios 10\%, 25\% and 50\% and the number of learnable network parameters (\#Par) using a prescribed compressed sensing matrix $\Abf$.  Subscript $*$ indicates that network parameters are shared across different phases. The \#Par of DR$^2$-Net \cite{yao2019dr2} and DPA-Net \cite{sun2020dual} reported below are calculated on CS ratio $25 \%$.
}
\label{tab:cs-result}
\medskip
\centering
\begin{tabular}{lcccr}
\toprule
\textbf{Method}  & \textbf{10\%}  & \textbf{25\%} & \textbf{50\%}& \textbf{\#Par}\\ 
\midrule 
TVAL3 \cite{li2013efficient} & 22.99 & 27.92& 33.55 & NA \\
D-AMP \cite{metzler2016denoising} & 22.64 & 28.46 &35.92 & NA \\ 
IRCNN \cite{zhang2017learning} & 24.02 & 30.07 & 36.23& 185,472\\ 
ReconNet \cite{kulkarni2016reconnet} & 24.28 & 25.60 & 31.50& \textbf{22,914}\\ 
DR$^2$-Net \cite{yao2019dr2} & 24.32 & 28.66 & - &373,664\\ 
ISTA-Net$_*^+$ \cite{Zhang2018ISTANetIO} & 26.51 & 32.08 & 37.59& 37,450\\ 
ISTA-Net$^+$ \cite{Zhang2018ISTANetIO} & 26.64 & 32.57 & 38.07& 336,978\\ 
DPA-Net \cite{sun2020dual} & 26.99 & 31.74 & 36.73& 9,519,750\\ 
\textbf{LDA} & \textbf{27.42} & \textbf{32.92} & \textbf{38.50}& {27,967}\\ 
\bottomrule
\end{tabular}
\end{table}
\begin{table}[H]
\centering
\caption{Average PSNR (dB) of reconstructions obtained by the compared methods and the proposed LDA on \textit{Set11} dataset with CS ratios 10\%, 30\% and the number of parameters (\#Par) in the reconstruction part of the network using jointly learned compressed sensing matrix $\Abf$.}
\begin{tabular}{lccr}
\toprule
\textbf{Method} & \textbf{10\%} & \textbf{30\%} & \textbf{\#Par} \\
\midrule 
CS-Net \cite{SJZ17} & 28.10 & 33.86 & 370,560\\
SCS-Net \cite{Shi_2019_CVPR} & 28.48 & 34.62 &  587,520\\ 
BCS-Net \cite{zhou2019multi} & 29.43 & 35.60 & 1,117,440\\ 
AMP-Net \cite{zhang2020amp} & 29.45 & 35.90 & 229,254\\
\textbf{LDA} & \textbf{30.03} & \textbf{36.47} &\textbf{27,967} \\ 
\bottomrule
\end{tabular}
\label{tab:jointcs}
\end{table}
%
\subsubsection{Joint compression and reconstruction of natural images}
We test LDA Algorithm \ref{alg:lda} on the problem of joint image compression and reconstruction, which is considered in several recent CS image reconstruction work \cite{Shi_2019_CVPR,zhang2020amp,zhou2019multi,ZZG20}. In this experiment, we prescribe the CS ratio $c \in (0,1)$ and let the compressed sensing matrix $\Abf \in \mathbb{R}^{cn\times n}$ be learned together with the reconstruction network. More precisely, we let a ground truth image (patch) $\xhat$ first pass a linear layer $\bbf= \Abf \xhat$, where $\Abf$ is also to be learned. Here $\Abf$ can be implemented as a convolutional operation with $cn$ kernels of size $\sqrt{n} \times \sqrt{n}$ and stride $\sqrt{n} \times \sqrt{n}$, and hence once applied to an image patch it returns a $cn$-vector. The sampling layer is followed by an initialization layer $\xbf_0 = \tilde{\Abf}\bbf$, where $\tilde{\Abf} \in \mathbb{R}^{n\times cn}$ is implemented as transposed convolutional operation \cite{dumoulin2016guide}. Then $\xbf_0$ is served as the input of LDA. Moreover, we add $(1/N)\cdot\sum_{s = 1}^N \|\tilde{\Abf}\Abf \xbf^{(s)} - \xbf^{(s)}\|^2$ with weight $0.01$ to the loss function in \eqref{eq:loa}, such that $\Abf$ and $\tilde{\Abf}$ are learned jointly with the network parameters during training.

The training dataset in our experiment consists of 89,600 image patches of size $96 \times 96$, where all these patches are the luminance components randomly cropped from images in BSD500 training and testing set ($200 + 200$ images) \cite{arbelaez2010contour}. 
Each image patch consists of 9 non-overlapping blocks of size $n=32\times 32=32^2$, where each block can be sampled independently by $\Abf$.
We use \textit{Set11} for testing. 
For comparison, we also test four recent methods in this experiment: CS-Net \cite{SJZ17}, SCS-Net \cite{Shi_2019_CVPR}, BCS-Net \cite{zhou2019multi} and AMP-Net \cite{zhang2020amp}.
All the compared methods are applied to \textit{Set11}, and the average PSNR are shown in Table \ref{tab:jointcs}.
Table \ref{tab:jointcs} also shows the number of learnable parameters of the reconstruction network part of each method. 
In addition to these parameters, all methods also need to learn the sampling matrix $\Abf$ with $cn \times n = 104,448$ variables when $c = 0.1$ and another $104,448$ variables of $\tilde{\Abf}$ for initialization, except that BCS-Net requires over 2.2M parameters for sampling and initialization. BCS-Net learns a set of sampling matrices with different rates and dynamically assigns the sampling resource depending on the embedded saliency information of each block \cite{zhou2019multi}.  
From Table \ref{tab:cs-result}, we can see that LDA outperforms all these state-of-the-art methods by a large margin, but only needs a fraction of the amount of learnable parameters compared to most methods.
\blue{
In Figures \ref{fig:butteryfly}--\ref{fig:parrot}, we show the reconstructed butterfly and cameraman images with CS ratio $10\%$ as well as the parrot image with CS ratio $30\%$. It is remarkable that the proposed LDA is able to recover fine patterns and details in the images, such as the texture of butterfly wing and the parrot feather shown in Figures \ref{fig:butteryfly} and \ref{fig:parrot}, and the boundary of the camera stand in Figure \ref{fig:lena}.
}
\begin{figure}[H]
\centering
\subfigure{\includegraphics[width=0.22\textwidth]{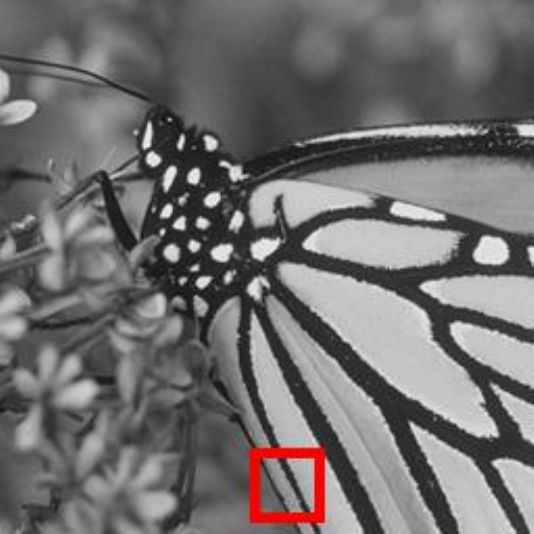}\label{gt_butterfly}}
\subfigure{\includegraphics[width=0.22\textwidth]{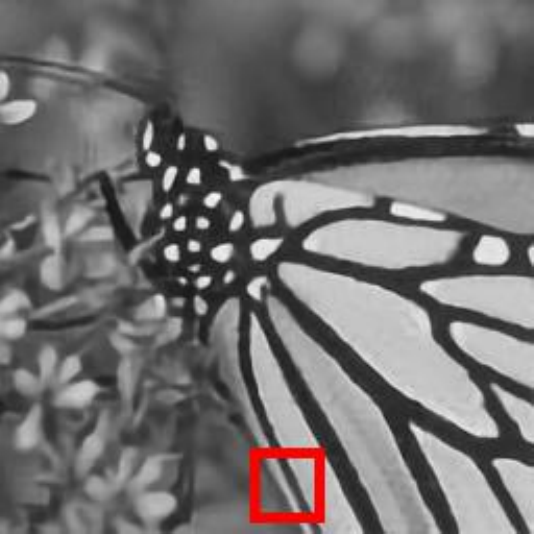}}
\subfigure{\includegraphics[width=0.22\textwidth]{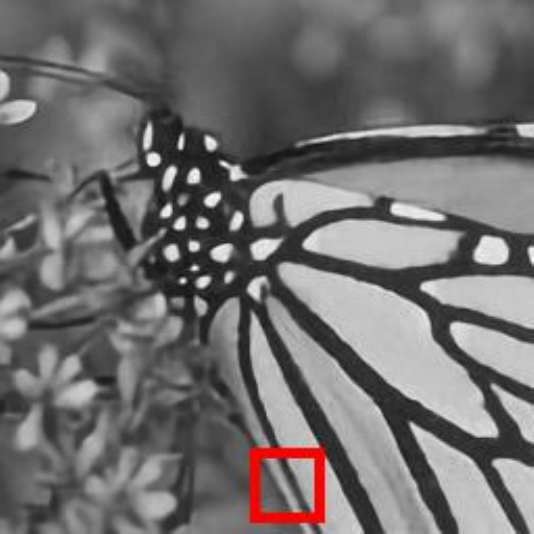}}
\subfigure{\includegraphics[width=0.22\textwidth]{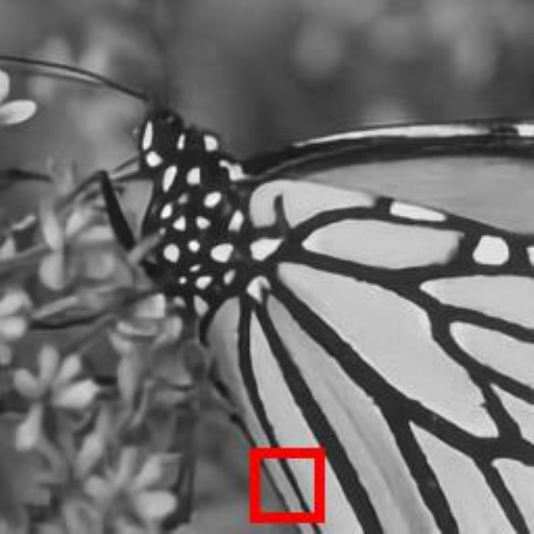}} \\
\setcounter{subfigure}{0}
\subfigure[\scriptsize{Reference}]{\includegraphics[width=0.22\textwidth]{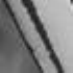}}
\subfigure[\scriptsize{CS-Net (28.31)}]{\includegraphics[width=0.22\textwidth]{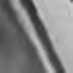}}
\subfigure[\scriptsize{SCS-Net (28.88)}]{\includegraphics[width=0.22\textwidth]{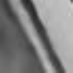}}
\subfigure[\scriptsize{LDA (30.54)}]{\includegraphics[width=0.22\textwidth]{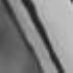}}
\caption{Block CS reconstruction of butterfly image with CS ratio 10\% obtained by CS-Net, SCS-Net and the proposed LDA. Images in the bottom row zoom in the corresponding ones in the top row. PSNR are shown in the parentheses.}
\label{fig:butteryfly}
\end{figure}
\begin{figure}[H]
\centering
\subfigure{\includegraphics[width=0.22\textwidth]{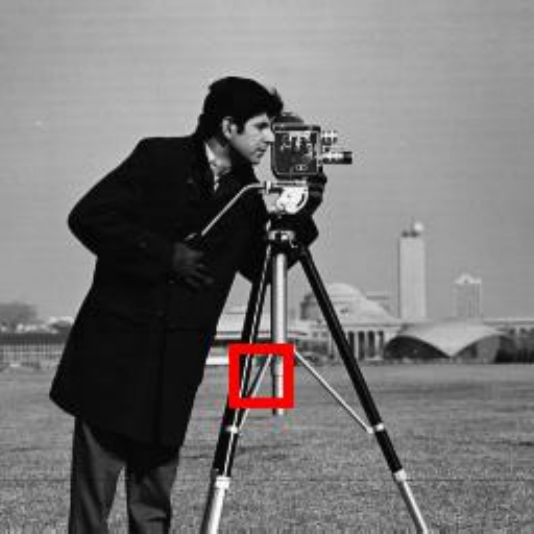}\label{gt_lena}}
\subfigure{\includegraphics[width=0.22\textwidth]{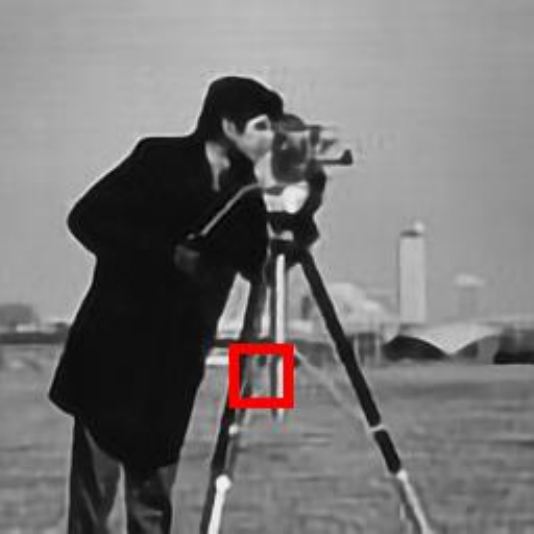}}
\subfigure{\includegraphics[width=0.22\textwidth]{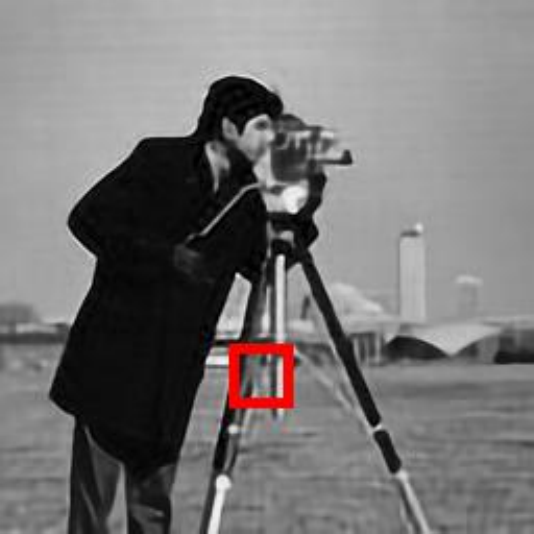}}
\subfigure{\includegraphics[width=0.22\textwidth]{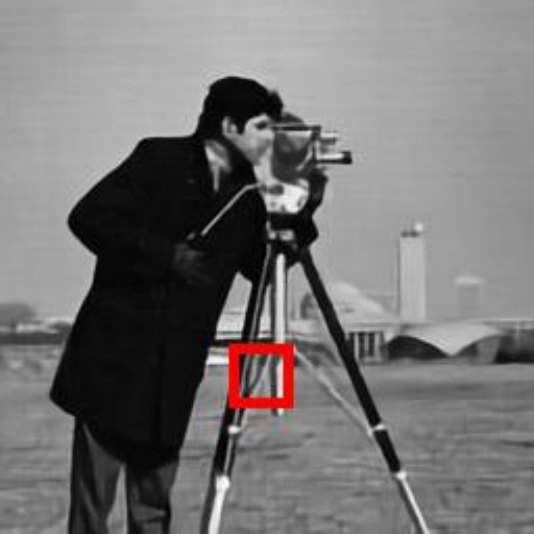}} \\
\setcounter{subfigure}{0}
\subfigure[\scriptsize{Reference}]{\includegraphics[width=0.22\textwidth]{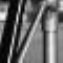}}
\subfigure[\scriptsize{CS-Net (25.35)}]{\includegraphics[width=0.22\textwidth]{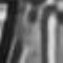}}
\subfigure[\scriptsize{SCS-Net (25.71)}]{\includegraphics[width=0.22\textwidth]{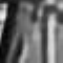}}
\subfigure[\scriptsize{LDA (27.45)}]{\includegraphics[width=0.22\textwidth]{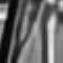}}
\caption{Block CS reconstruction of a cameraman image with CS ratio 10\% obtained by CS-Net, SCS-Net and the proposed LDA. Images in the bottom row zoom in the corresponding ones in the top row. PSNR are shown in the parentheses.}
\label{fig:lena}
\end{figure}
\begin{figure}[H]
\centering
\subfigure{\includegraphics[width=0.22\textwidth]{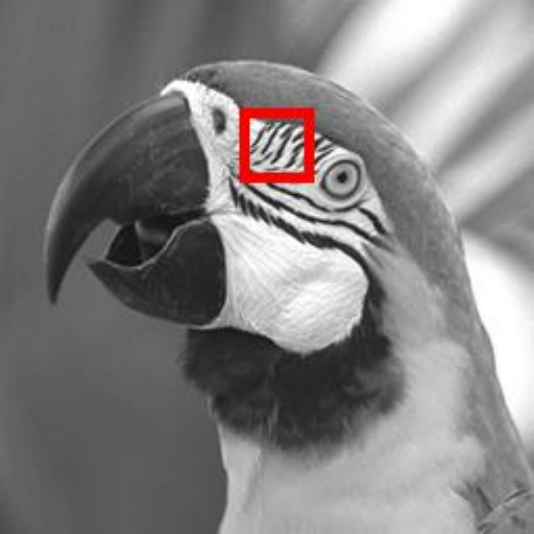}\label{gt_parrot}}
\subfigure{\includegraphics[width=0.22\textwidth]{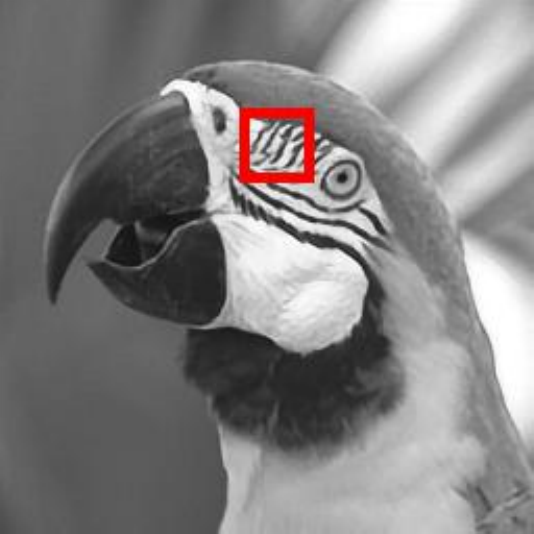}}
\subfigure{\includegraphics[width=0.22\textwidth]{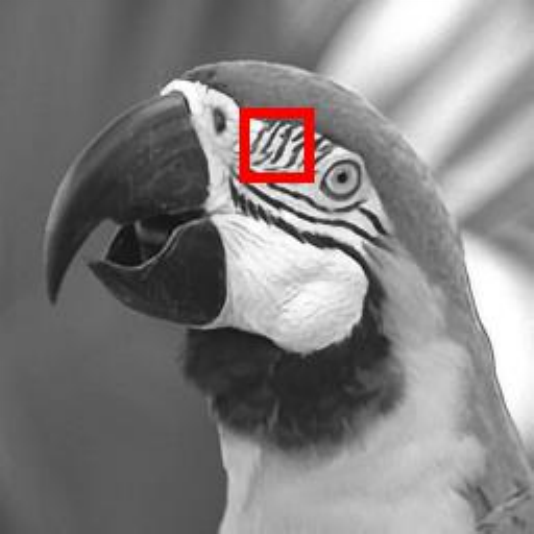}}
\subfigure{\includegraphics[width=0.22\textwidth]{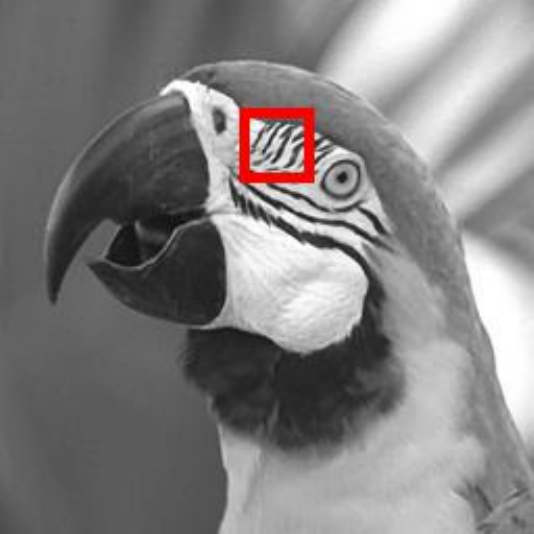}} \\
\setcounter{subfigure}{0}
\subfigure[\scriptsize{Reference}]{\includegraphics[width=0.22\textwidth]{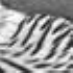}}
\subfigure[\scriptsize{CS-Net (33.77)}]{\includegraphics[width=0.22\textwidth]{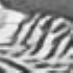}}
\subfigure[\scriptsize{SCS-Net (34.13)}]{\includegraphics[width=0.22\textwidth]{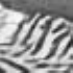}}
\subfigure[\scriptsize{LDA (36.43)}]{\includegraphics[width=0.22\textwidth]{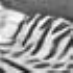}}
\caption{Block CS reconstruction of parrot image with CS ratio 30\% obtained by CS-Net, SCS-Net and the proposed LDA. Images in the bottom row zoom in the corresponding ones in the top row. PSNR are shown in the parentheses.}
\label{fig:parrot}
\end{figure}
\clearpage
\subsubsection{Magnetic resonance image reconstruction}
In this experiment, we consider the reconstruction problem in compressed sensing magnetic resonance imaging (CS-MRI). In CS-MRI, we set $\Abf = \mathcal{P}\mathcal{F}$, where $\mathcal{P}$ is a binary selection matrix representing the Fourier space ($k$-space) sampling trajectory, and $\mathcal{F}$ is the discrete Fourier transform. The ground truth image is shown in Figure \ref{fig:mri_rec_err}(a). We use radial mask $\mathcal{P}$ with three different sampling ratios  $10 \%$, $20 \%$ and $30 \%$ in this experiments. The one with 10\% sampling ratio is shown in Figure \ref{fig:mri_rec_err}(a). We randomly select $150$ 2D images from the brain MRI datasets \cite{mridata}, then extract the main center region of interests (size $190\times190$) of every image as the ground truth images $\xhat$, and set the data to $\bbf = \Abf \xhat$. Then we randomly select $100$ images for training and use the other $50$ for testing.

During the training, for each of the sampling ratios 10\%, 20\%, and 30\%, we train LDA for phase numbers $K=3,5,\dots,11$, and the PSNR obtained for each case is shown in the left panel of Figure \ref{fig:6}.
\blue{For comparison, we also apply Zero-filling \cite{bernstein2001effect}, ADMM-Net \cite{NIPS2016_6406}, ISTA-Net$^+$ \cite{Zhang2018ISTANetIO} and Variational Network (VN) \cite{HKK18} to the same data. We use $\xbf_0 = \mathbf{0}$ as the initial for ADMM-Net, ISTA-Net$^+$ and LDA, whereas VN takes the result of zero-filling as initial.}
The quantitative comparison results are shown in Table \ref{www}, where the PSNR, the relative error (RelErr) of the reconstruction $\xbf$ to the ground truth $\xhat$ defined by $\|\xbf - \xhat\|/\|\xhat\|$, and the structural similarity index (SSIM) \cite{wang2004image} are provided for each of the three sampling ratios.
\blue{
%
These results show that LDA generates more accurate images using relatively much fewer network parameters.
Figure \ref{fig:mri_rec_err} shows the reconstructed images obtained by ADMM-Net, ISTA-Net$^+$, VN and LDA under sampling ratio 10\%, as well as the corresponding pointwise absolute error where brighter pixels indicate larger errors. From Figure \ref{fig:mri_rec_err}, it can be seen that LDA attains much lower error and better reconstruction quality. 
}

\begin{table}[htb]
\centering
\caption{Average PSNR (dB), RelErr, and SSIM of the reconstructions obtained by ADMM-Net, ISTA-Net$^+$, Variational Network (VN) and LDA on CS-MRI dataset with sampling ratios 10\%, 20\%, and 30\% and the number of learnable network parameters (\#Par).}
\label{www}
\resizebox{\textwidth}{16.5mm}{
\begin{tabular}{ccccccccccr}
\toprule
\multirow{2}{*}{\textbf{Method}} & \multicolumn{3}{c}{\textbf{10\%}} & \multicolumn{3}{c}{\textbf{20\%}} & \multicolumn{3}{c}{\textbf{30\%}} & \multirow{2}{*}{\textbf{\#Par}}\\
& PSNR & RelErr & SSIM & PSNR & RelErr & SSIM & PSNR & RelErr & SSIM \\
\midrule
Zero-filling \cite{bernstein2001effect} &23.45 & 0.2820 & 0.4544 & 27.46 & 0.1806 & 0.5820 & 30.91 & 0.1232 & 0.6665 & NA\\
ADMM-Net \cite{NIPS2016_6406} & 30.43 & 0.1193 & 0.7990 & 37.73 & 0.0516 & 0.9507 & 41.90 & 0.0321 & 0.9646 & \textbf{14,600}\\
ISTA-Net$^+$ \cite{Zhang2018ISTANetIO} & 32.62 & 0.0950&0.9312 & 39.84 & 0.0430&0.9816 & 43.53 & 0.0295&0.9892 & 823,692\\
VN \cite{HKK18} & 33.21 & 0.0882 & 0.9395 & 40.11 & 0.0400 & \textbf{0.9855} & 44.27 & 0.0249 & 0.9928 & 131,050 \\
\textbf{LDA} & \textbf{34.20} & \textbf{0.0790}& \textbf{0.9462}& \textbf{41.03} & \textbf{0.0363}&0.9852 & \textbf{46.12} &\textbf{0.0214} & \textbf{0.9931} & 55,895\\
\bottomrule
\end{tabular}
}
\end{table}
\begin{figure}[htb]
\centering
\subfigure{\includegraphics[width=0.192\textwidth]{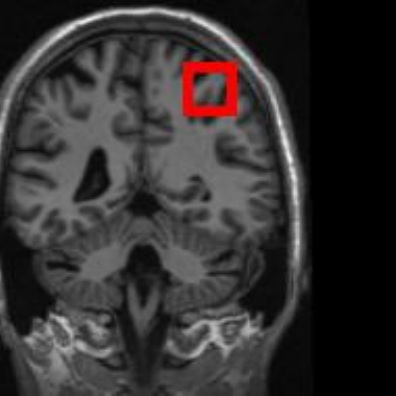}}
\subfigure{\includegraphics[width=0.192\textwidth]{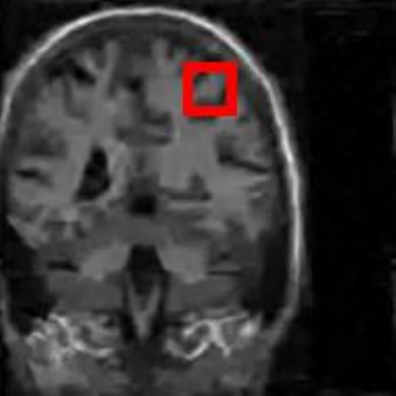}}
\subfigure{\includegraphics[width=0.192\textwidth]{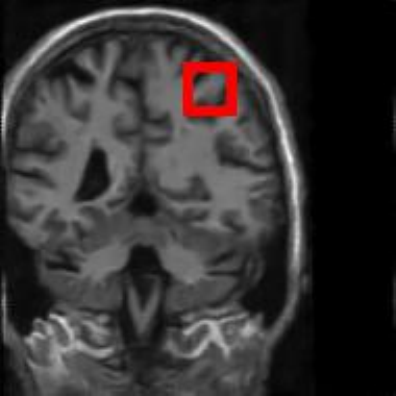}}
\subfigure{\includegraphics[width=0.192\textwidth]{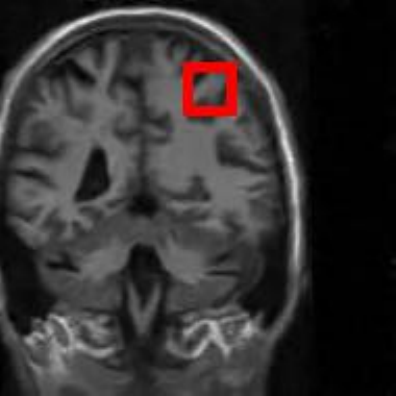}}
\subfigure{\includegraphics[width=0.192\textwidth]{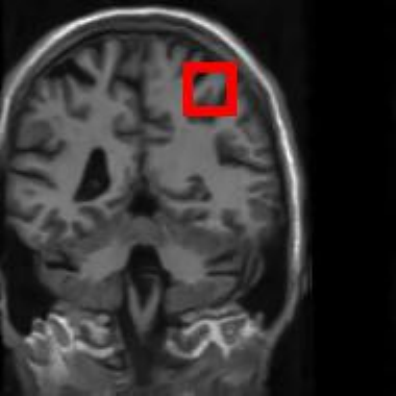}}\\
\subfigure{\includegraphics[width=0.192\textwidth]{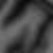}}
\subfigure{\includegraphics[width=0.192\textwidth]{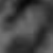}}
\subfigure{\includegraphics[width=0.192\textwidth]{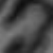}}
\subfigure{\includegraphics[width=0.192\textwidth]{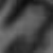}}
\subfigure{\includegraphics[width=0.192\textwidth]{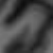}}
\setcounter{subfigure}{0}
\subfigure[Reference \& mask]{\includegraphics[width=0.192\textwidth]{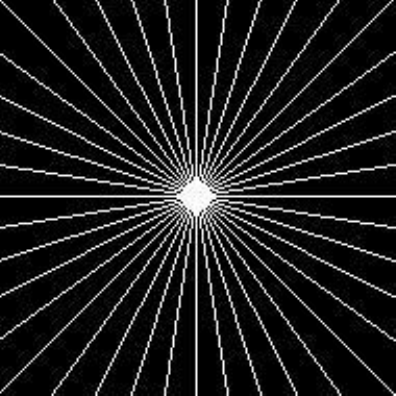}}
\subfigure[ADMM-Net (28.34) ]{\includegraphics[width=0.192\textwidth]{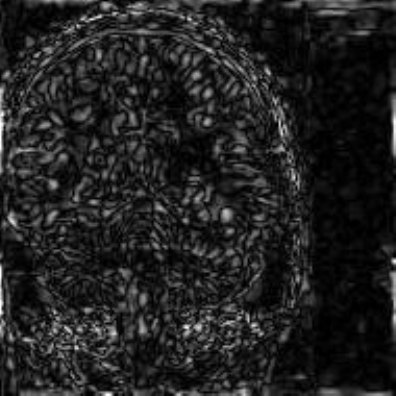}}
\subfigure[ISTA-Net$^+$ (30.25) ]{\includegraphics[width=0.192\textwidth]{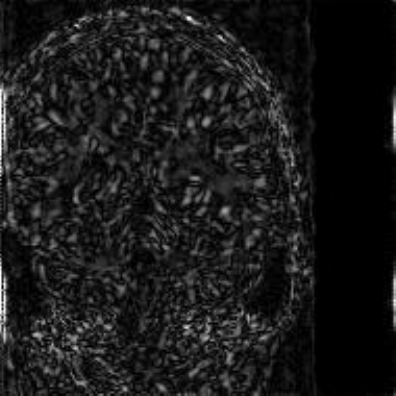}}
\subfigure[VN (31.68) ]{\includegraphics[width=0.192\textwidth]{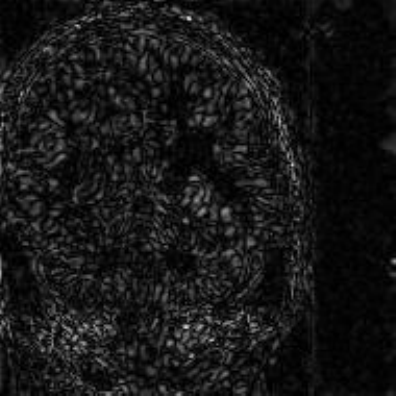}}
\subfigure[LDA (32.78) ]{\includegraphics[width=0.192\textwidth]{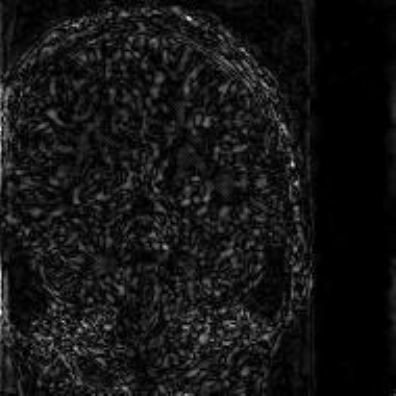}}
\caption{
Representative brain MR images reconstructed by ADMM-Net \cite{NIPS2016_6406}, ISTA-Net$^+$ \cite{Zhang2018ISTANetIO}, VN \cite{HKK18} and the proposed LDA  with CS ratio 10\%. Images in the middle row magnify the corresponding regions of interest in the top row. 
Pointwise absolute errors of the reconstruction in the bottom row are rescaled by the same level for better visualization.
 Brighter pixel indicates larger value. PSNR are shown in the parentheses. Ground truth reference image and a radial mask with 10\% sampling ratio are shown in the first column.}
\label{fig:mri_rec_err}
\end{figure}
\subsection{Experimental results on convergence, parameters and learned feature map}
\subsubsection{Comparison with standard gradient descent}
The proposed LDA performs two residual-type updates, one on the data fidelity $f$ and the other on the regularization $r$ (and the smoothed version $\reps$), which is motivated by the effectiveness of the ResNet structure. In this experiment, we also unroll the standard gradient descent iteration by turning off the $\ubf$ step of LDA, and an accelerated inertial version by setting $\xbf_{k+1} = \xbf_k - \alpha_k \nabla f(\xbf_k) + \theta_k (\xbf_k - \xbf_{k-1})$ where $\theta_k$ is also learned. We call these two networks GD-Net and AGD-Net, respectively. 
We test all three methods following the same experiment setting in Section \ref{subsubsec:cs}, and show the average PSNR versus phase (iteration) number of these methods in Figure \ref{fig:6}. As we can see, LDA achieves a much higher PNSR than both GD-Net and AGD-Net, where the latter perform very similarly. In particular, although AGD has improved iteration complexity in the standard convex optimization setting, its network version does not seem to inherit the effectiveness for deep learning applications. Similar comparison has been made for ISTA-Net and FISTA-Net, which are based on ISTA and FISTA with the latter algorithm provably having improved iteration complexity, but their deep network versions have nearly identical performance \cite{zhang2017learning}. This is also partly due to the nonconvexity of the learned objective function, for which inertial gradient descent may produce improper extrapolation and do not improve efficiency.

\begin{figure}[h]
\centering
\includegraphics[width=0.346\textwidth]{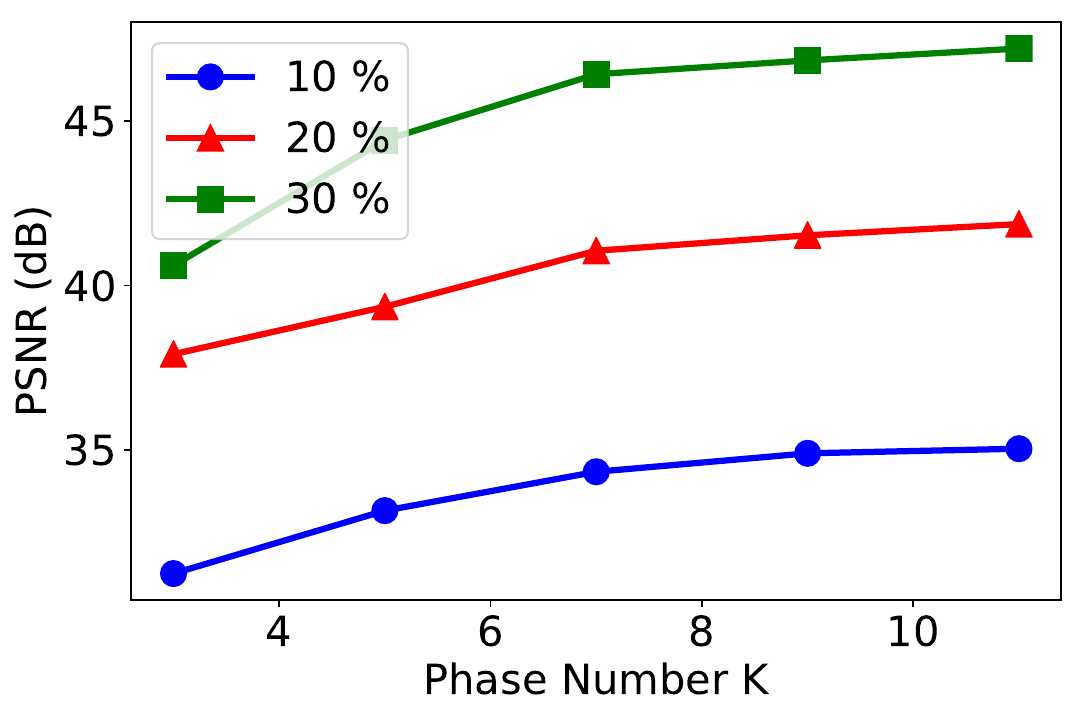}
\includegraphics[width=0.360\textwidth]{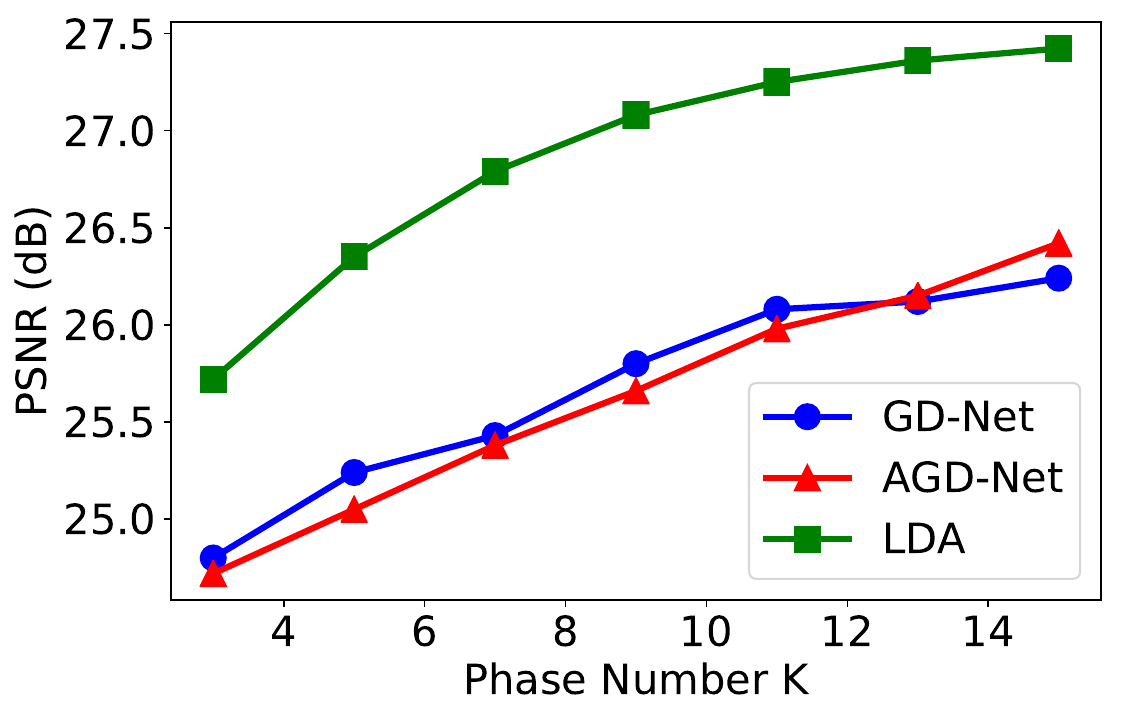}
\caption{\emph{Left}: PSNR of reconstructions obtained by LDA versus phase number $K$ on three CS ratios $10 \%, 20 \%$ and $ 30 \%$ for the brain MR image.
\emph{Right}: PSNR of reconstructions obtained by GD-Net, AGD-Net, and LDA versus phase number $K$ on the block CS image reconstruction with CS ratio $10\%$.
}
\label{fig:6}
\includegraphics[width=0.357\textwidth]{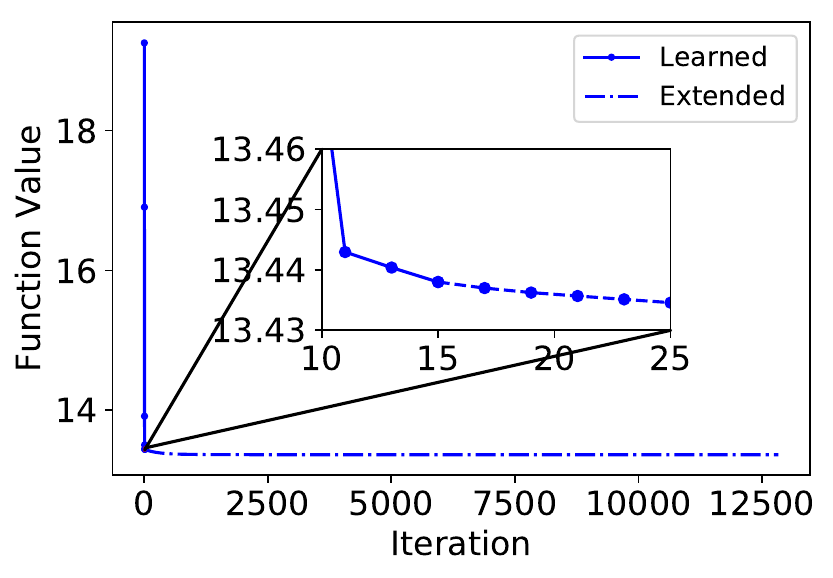}
\includegraphics[width=0.357\textwidth]{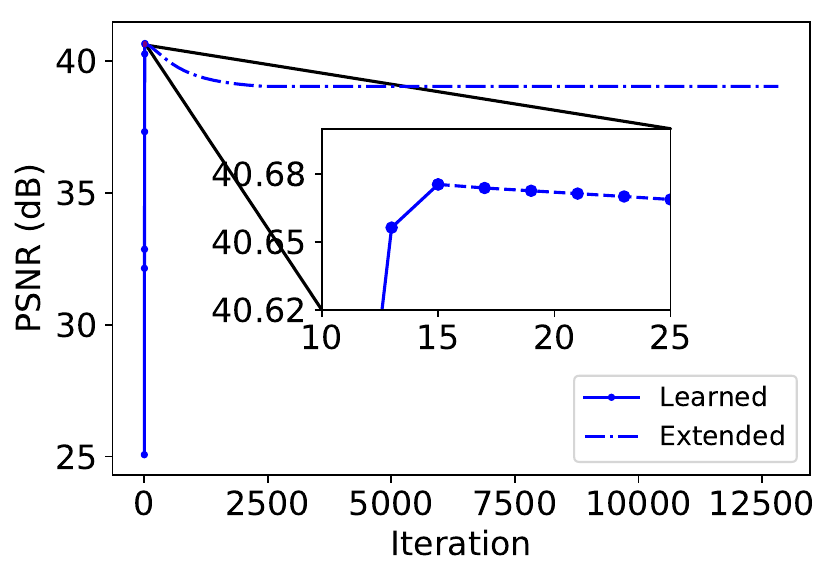}
\caption{
\newblue{Convergence behavior of LDA on compressed image reconstruction on test image with CS ratio $50 \%$ using learned regularization $r$. LDA uses learned algorithm parameters in the first 15 iterations and then follows Algorithm \ref{alg:lda} until the termination criterion in Step 8 is met. The results of the first 15 iterations are plotted in solid lines and those of the other iterations are plotted in dotted lines. \emph{Left}: Objective function value $\phi(\xbf_k)$ versus iteration number $k$. \emph{Right}: PSNR versus iteration number $k$. }
}
\label{fig:7}

\end{figure}

\blue{
\subsubsection{Trade-offs between network performance and complexity}
The regularization term of LDA is learned from training samples, yet there are still a few key network hyperparameters that have to be set manually. Specifically, we investigate the effects of several main factors of the network architecture, including the number of convolutions ($l$) and the depth of the convolution kernels ($d$). We set the default value as $d = 32$ and $l = 4$, and test the performance by varying one of them while keeping the other one as default. For fair comparison, we keep the phase number $K = 7$ for each experiment in this study. The experiment setting and datasets are identical to Section \ref{subsubsec:cs}, and all following results are trained and tested with CS ratio $10 \%$.

We first consider the effect of $d$. We evaluate the instances of $d = 8, 16, 32$ and $48$ respectively. The results are listed and compared in Table \ref{tab:depth}. As expected, the PSNR improves with larger $d$ and increased representation power, but the margin gradually decreases. On the contrary, the number of parameters and running time grow significantly in the meantime. It seems that $d=32$ is a good compromise between network complexity and reconstruction quality, which is also the value we used in LDA in this work.

Next we consider the effect of $l$. We evaluate the cases of different number of convolutions $l = 2, 4$ and $6$. The corresponding tested results are reported in Table \ref{tab:filternum}. Again the PSNR increases with larger $l$. However PSNR only improves slightly when $d$ increases from $4$ to $6$, whereas the parameter number and the testing time also increase significantly. Therefore $l=4$ of LDA appears to be a good balance, which is the value we set for LDA in our other tests. 
\begin{figure}[H]
\centering
\subfigure[Reference]{\includegraphics[width=0.155\textwidth]{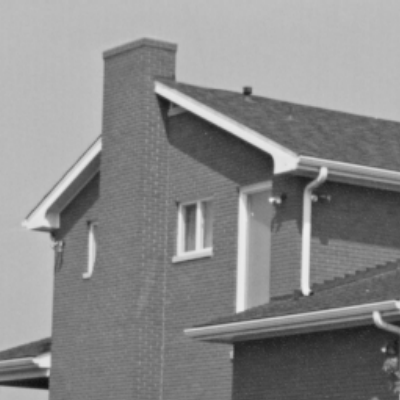}}
\subfigure[Iter 15 (40.68)]{\includegraphics[width=0.155\textwidth]{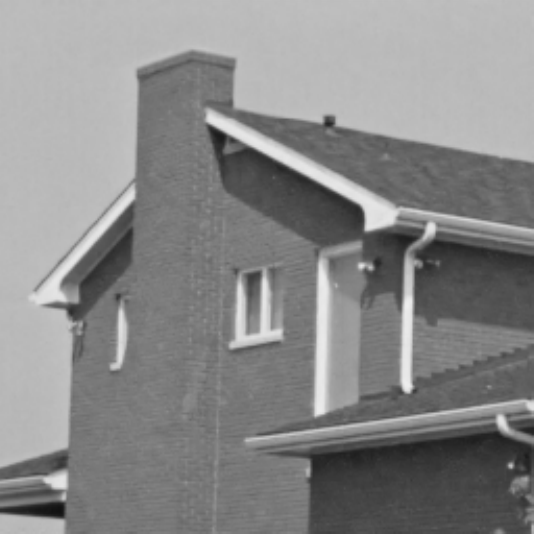}}
\subfigure[Iter 500 (39.99) ]{\includegraphics[width=0.155\textwidth]{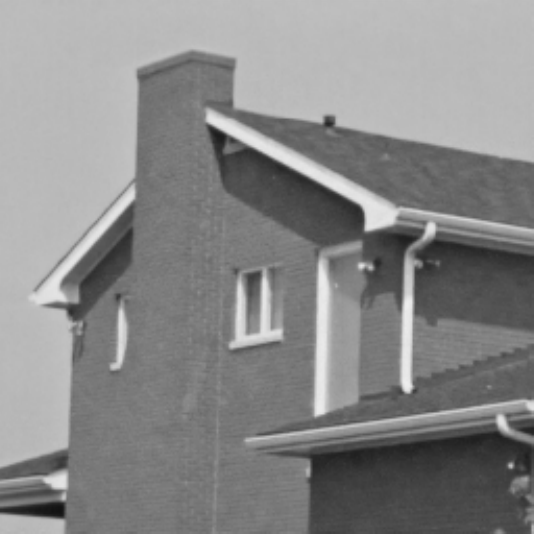}}
\subfigure[Iter 1K (39.49)]{\includegraphics[width=0.155\textwidth]{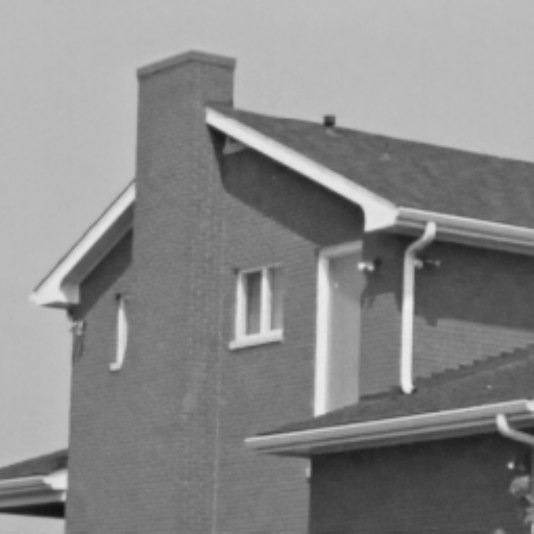}}
\subfigure[Iter 5K (39.07) ]{\includegraphics[width=0.155\textwidth]{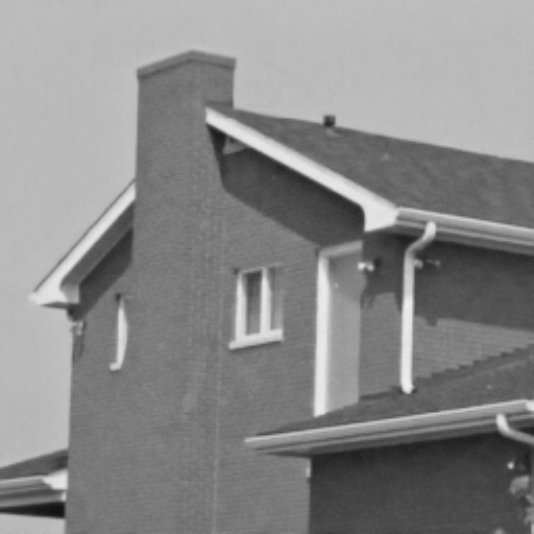}}
\subfigure[Iter 12K (39.05) ]{\includegraphics[width=0.155\textwidth]{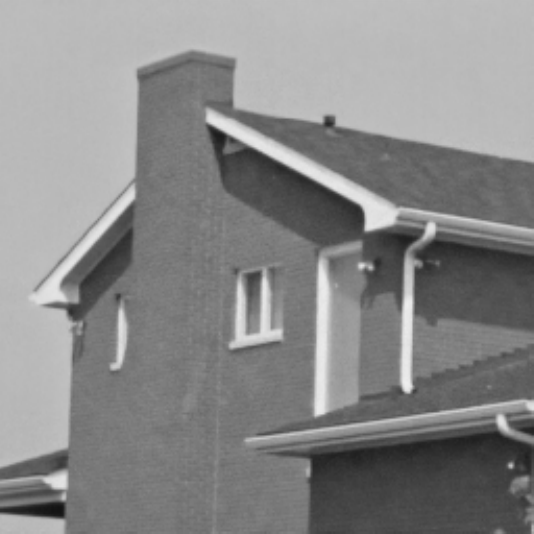}}
\caption{\newblue{Reconstructed \textit{House} images obtained by LDA with CS ratio $50\%$ after 15, 500, 1K, 5K and 12K iterations. PSNR values are given in the parentheses.}}
\label{fig:8}

\end{figure}

\begin{table}[H]
\centering
\caption{
The results of reconstruction associated with different depths of convolution kernels on \textit{Set11} dataset with CS ratio 10\%. The phase number here is set to be 7.
}
\begin{tabular}{lcccc}
\toprule
\textbf{Depth of conv. kernels} & \textbf{8} & \textbf{16} & \textbf{32} & \textbf{48}\\
\midrule 
PSNR (dB) & 25.60 & 26.36 & 26.79 & 26.88 \\
Number of parameters & 1,815 & 7,071 & 27,951 & 62,655 \\
Average testing time (s) & 0.035 & 0.059 & 0.106 & 0.171 \\
\bottomrule
\end{tabular}
\label{tab:depth}
\end{table}
\begin{table}[H]
\centering
\caption{
The results of reconstruction associated with different numbers of convolutions in each phase on \textit{Set11} dataset with CS ratios 10\%. The phase number here is set to be 7.
}
\begin{tabular}{lccc}
\toprule
\textbf{Number of convolutions} & \textbf{2} & \textbf{4} & \textbf{6}\\
\midrule 
PSNR (dB) & 25.83 & 26.79 & 26.95 \\
Number of parameters & 9,519 & 27,951 & 46,383\\
Average testing time (s) & 0.058 & 0.106 & 0.164 \\
\bottomrule
\end{tabular}
\label{tab:filternum}
\end{table}
}
\subsubsection{Convergence behavior of LDA}
As in the standard approach of deep neural network training, we set the phase (iteration) number to $K=15$ in LDA in the experiments above. On the other hand, we proved that the iterates generated by LDA converge to a Clarke stationary point in Section \ref{subsec:convergence}. This provides a theoretical guarantee that the LDA is indeed minimizing an objective function where the regularization is learned, and LDA is expected to perform stably even beyond the trained phases. 
\newblue{To demonstrate this stability empirically, we set $\epsilon_{\text{tol}} = 1.5 \times 10^{-5}$, $\sigma = 10^{3}$ and $\gamma = 0.9$ in LDA and let it continue to run after the initial 15 iterations (we use trained parameters in the first 15 iterations) on the test image \textit{House} in \textit{Set11}. We set $\varepsilon_0 = 1.4 \times 10^{-3}$ which is obtained after learning. LDA automatically reduces $\varepsilon_k$ and computes the step sizes $\alpha_k$ using the standard line search backtracking with step size reduction rate $0.5$ such that $\phi_{\epsk}(\vkp) - \phi_{\epsk}(\xbf_k) \le - \tau \| \vkp - \xbf_k\|^2$ with $\tau$ set to  $0.35$.
Under this setting, the termination criterion in Step 8 of Algorithm \ref{alg:lda} is met when total iterations $k = 12,832$ and $\varepsilon_k$ {is reduced} $109$ times. The changes of objective function value and PSNR in iteration number $k$ are shown in Figure \ref{fig:7}.
The objective function value $\phi(\xbf_k)$ versus iteration $k$ is shown in the left panel of Figure \ref{fig:7} and the corresponding PSNR is shown in the right panel. As we can see, LDA continues to reduce function value during these extended iterations, as shown in our convergence analysis in Section \ref{subsec:convergence}.
The PSNR slightly drops after the trained 15 iterations and remains steady since after $2,500$ iterations. The reconstructed images during the iterations seem to be very similar without any visual artifacts despite of the slight drop of PSNR. 
In Figure \ref{fig:8}, we show the reconstructed images obtained after $15$, $500$, $1000$, $5000$ and $12000$ iterations, which appear to be very similar without any visual artifacts.
}

\subsubsection{Learned feature map in LDA}
A main advantage of LDA is that the feature map $\gbf$ in the regularization can be learned adaptively from the training data such that the network is more interpretable. This data-driven approach yields automated design of feature maps which are often more complex and efficient, rather than the manually crafted features in the classical image reconstruction models.

In this experiment, we plot the norm of the gradient (as a 2D vector computed by forward finite difference) at every pixel $i$ which is used as the feature of the TV based image reconstruction and also $\|\gbf_i(\xbf)\|$ at pixel $i$ of the regularization learned in LDA in Figure \ref{fig:vs_tv}. We can see that the learned feature map $\gbf$ captures more important structural details of the images, such as the antennae of the butterfly, the buildings behind the cameraman, and the bill of the parrot. These details are crucial in species detection and facial recognition, 
which seem to be accurately recovered using the learned feature map $\gbf$ but are heavily blurred or completely missing from the simple gradient image used in TV regularization. This also explains the better image quality obtained by LDA compared to the classical TV based image reconstruction methods.
\begin{figure}[H]
\centering
\includegraphics[width=0.25\textwidth]{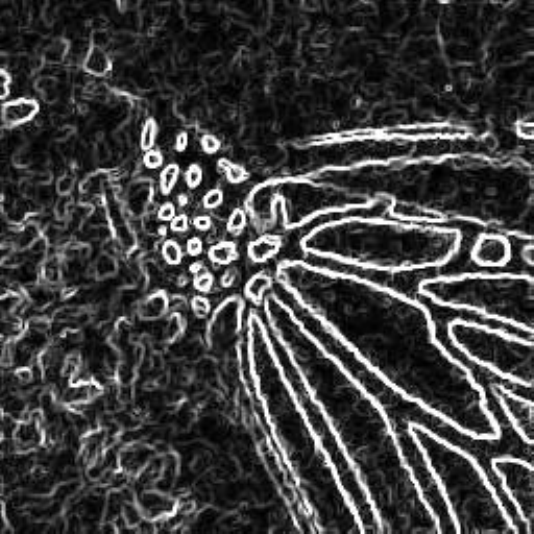}
\includegraphics[width=0.25\textwidth]{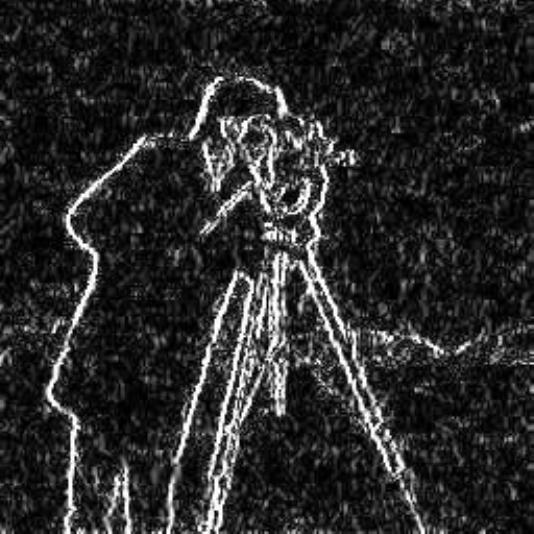}
\includegraphics[width=0.25\textwidth]{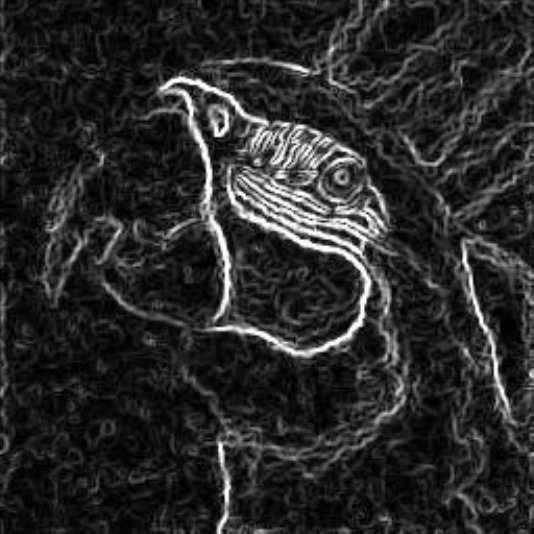}
\includegraphics[width=0.25\textwidth]{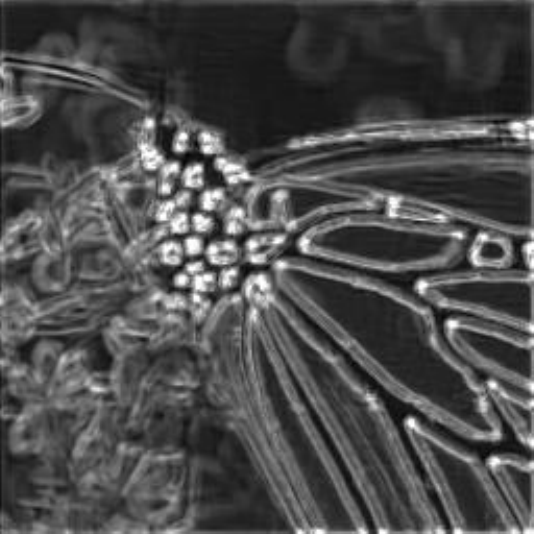}
\includegraphics[width=0.25\textwidth]{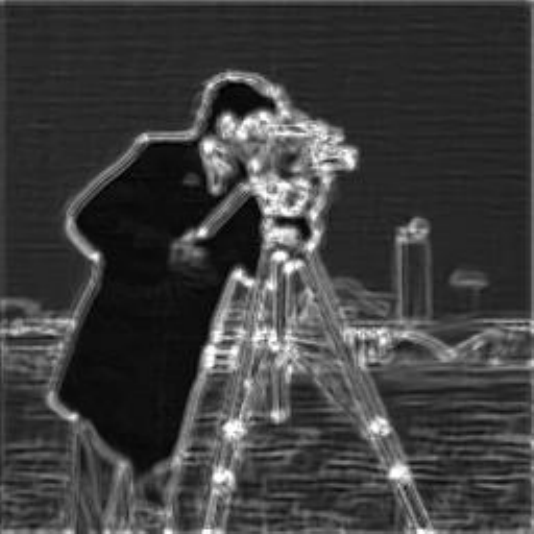}
\includegraphics[width=0.25\textwidth]{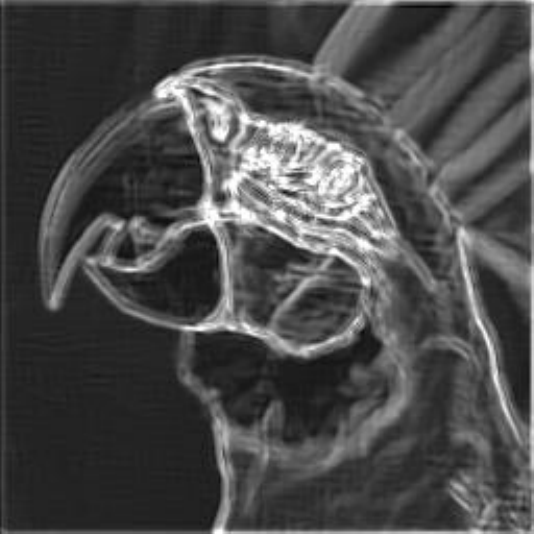}
\caption{The norm of the gradient at every pixel in TV based image reconstruction (top row) and the norm of the feature map $\gbf$ at every pixel learned in LDA (bottom row) when CS ratio 10\%. Important details, such as the antennae of the butterfly, the buildings behind the cameraman, and the bill of the parrot, are faithfully recovered by LDA.}
\label{fig:vs_tv}
\end{figure}
\section{Conclusion}
\label{sec:conclusion}
We proposed a general learning based framework for solving nonsmooth and nonconvex image reconstruction problems, where the regularization function is modeled as the composition of the $l_{2,1}$ norm and a smooth but nonconvex feature mapping parametrized as a deep convolutional neural network. We developed a descent-type algorithm to solve the nonsmooth nonconvex minimization problem by leveraging Nesterov's smoothing technique and the idea of residual learning, and learn the network parameters such that the outputs of the algorithm match the references in training data. Our method is versatile as one can employ various modern network structures into the regularization, and the resulting network inherits the guaranteed convergence of the algorithm. The proposed network is applied to a variety of real-world image reconstruction problems, and the numerical results demonstrate the outstanding performance and efficiency of our method.

\section*{Acknowledgments}
This research was partially supported by NSF grants DMS-1319050, DMS-1719932, DMS-1818886, DMS-1925263, CMMI-2016571 and University of Florida AI Catalyst Grants.

\bibliographystyle{abbrv}
\bibliography{egbib}

\end{document}